\documentclass{article}

\PassOptionsToPackage{numbers,compress}{natbib}
\usepackage{iclr2026_conference,times}

\usepackage[T1]{fontenc}
\usepackage{microtype}
\usepackage{url}

\usepackage{amsmath,amssymb,mathtools,amsthm}
\usepackage{algorithm}
\usepackage{algpseudocode}

\usepackage{graphicx}
\usepackage{subcaption}
\usepackage{booktabs}
\usepackage{multirow}
\usepackage{array}
\usepackage{tabularx}
\usepackage{longtable}
\usepackage{adjustbox}
\usepackage{wrapfig}

\usepackage{enumitem}
\usepackage[table]{xcolor}
\usepackage[most]{tcolorbox}

\usepackage{titletoc}

\usepackage{hyperref}
\usepackage[capitalize,noabbrev]{cleveref}

\usepackage{pifont}
\definecolor{ICMLBlue}{rgb}{0.0, 0.0, 0.5}
\definecolor{CiteTeal}{HTML}{008080}
\definecolor{lightcrimson}{rgb}{0.93, 0.16, 0.51}
\definecolor{LinkRed}{rgb}{0.768, 0.054, 0.054}

\hypersetup{
    colorlinks=true,
    linkcolor=ICMLBlue,      % sections, appendix TOC, \ref, \autoref
    citecolor=ICMLBlue,      % \citep, \citet, \cite
    urlcolor=ICMLBlue,   % \url, \href
    linktoc=page
}

\definecolor{lightgray}{gray}{0.9}
\definecolor{lightblue}{RGB}{230,245,255}

\theoremstyle{plain}
\newtheorem{theorem}{Theorem}[section]

\newtheorem{problem}{Problem}

\theoremstyle{definition}
\newtheorem{definition}[theorem]{Definition}

\theoremstyle{remark}

\titlecontents{section}[1.5em]
    {\addvspace{6pt}\bfseries}
    {\contentslabel{1.5em}}
    {\hspace*{-1.5em}}
    {\titlerule*[0.5pc]{.}\contentspage}

\titlecontents{subsection}[3.8em]
    {}
    {\contentslabel{2.3em}}
    {\hspace*{-2.3em}}
    {\titlerule*[0.5pc]{.}\contentspage}

\newcommand{\appendixtoc}{
    \begingroup
    \renewcommand{\contentsname}{\large\textbf{Appendix}}
    \startcontents[appendices]
    \printcontents[appendices]{}{1}{\setcounter{tocdepth}{2}}
    \endgroup
}

\newcommand{\std}[1]{\textcolor{gray}{\fontsize{5}{6}\selectfont$\pm$#1}}

\makeatletter
\@ifpackageloaded{algorithm}{
    
}{}
\makeatother

\tcbuselibrary{breakable,skins,listings}

\definecolor{BoxGrayBack}{HTML}{F8F8F8}
\definecolor{BoxGrayFrame}{HTML}{B8B8B8}
\definecolor{BoxBlueBack}{HTML}{F5F9FF}
\definecolor{BoxBlueFrame}{HTML}{2D5DA8}
\definecolor{BoxGreenBack}{HTML}{F6FFF8}
\definecolor{BoxGreenFrame}{HTML}{2E7D32}

\newtcolorbox{detailbox}[2][]{
  enhanced,
  breakable,
  colback=BoxGrayBack,
  colframe=BoxGrayFrame,
  boxrule=0.5pt,
  arc=2pt,
  left=6pt,
  right=6pt,
  top=5pt,
  bottom=5pt,
  fonttitle=\bfseries,
  coltitle=black,
  title={#2},
  #1
}

\newtcblisting{promptbox}[2][]{
  enhanced,
  breakable,
  listing only,
  colback=BoxBlueBack,
  colframe=BoxBlueFrame,
  boxrule=0.6pt,
  arc=2pt,
  left=6pt,
  right=6pt,
  top=5pt,
  bottom=5pt,
  fonttitle=\bfseries,
  coltitle=white,
  colbacktitle=BoxBlueFrame,
  title={#2},
  listing options={
    basicstyle=\ttfamily\footnotesize,
    breaklines=true,
    breakatwhitespace=false,
    columns=fullflexible,
    keepspaces=true,
    showstringspaces=false,
    upquote=true
  },
  #1
}

\newtcolorbox{implementationbox}[2][]{
  enhanced,
  breakable,
  colback=BoxGreenBack,
  colframe=BoxGreenFrame,
  boxrule=0.6pt,
  arc=2pt,
  left=6pt,
  right=6pt,
  top=5pt,
  bottom=5pt,
  fonttitle=\bfseries,
  coltitle=white,
  colbacktitle=BoxGreenFrame,
  title={#2},
  #1
}

\newtcolorbox{takeawaybox}{
  colback=red!4,
  colframe=red!60!black,
  boxrule=0.65pt,
  arc=1.5pt,
  left=6pt,
  right=6pt,
  top=5pt,
  bottom=5pt
}

\hypersetup{
  pdftitle={Adapting Context Compression for Long-Horizon Agents with Counterfactual Continuations},
  pdfauthor={Guanghui Min, Liang Wu, Minjia Shi, Yinhan He, Mayank Darbari, Liangjie Hong, Chen Chen}
}

\title{Adapting Context Compression for Long-Horizon Agents with Counterfactual Continuations}
\iclrfinalcopy

\author{
\textbf{Guanghui Min}$^{1}$ \quad
\textbf{Liang Wu}$^{2}$ \quad
\textbf{Minjia Shi}$^{1}$ \quad
\textbf{Yinhan He}$^{1}$ \\
\textbf{Mayank Darbari}$^{2}$ \quad
\textbf{Liangjie Hong}$^{2}$ \quad
\textbf{Chen Chen}$^{1}$\thanks{Corresponding author.} \\[0.6em]
$^{1}$University of Virginia
\qquad
$^{2}$Nokia \\[0.3em]
\texttt{\{jjm8vr,nzh3ru,nee7ne,zrh6du\}@virginia.edu} \\
\texttt{\{liang.wu,mayank.darbari,liangjie.hong\}@nokia.com}
}

\begin{document}

\maketitle

\begin{abstract}
  
Long-horizon agents require context compression to manage growing interaction histories.
Compression quality, however, is ultimately determined by downstream execution.
Existing prompt-adaptation methods infer compression errors by comparing full-context and compressed trajectories.
Such comparisons cannot isolate individual compressions and are confounded by agent stochasticity.
We first find that compression degrades reliability before solvability.
Using matched counterfactual continuations that compare execution from the same agent state with versus without compression, we further
show that severe degradation concentrates at isolated compression
events.
Motivated by this finding, we propose \texttt{PAIR} (\textbf{\underline{P}}rompt \textbf{\underline{A}}daptation using \textbf{\underline{I}}nterventional \textbf{\underline{R}}ollouts) for adapting structured compression prompts.
\texttt{PAIR} identifies individual compressions that degrade subsequent execution, diagnoses their effects, and revises the relevant sections of a fixed compression template.
\texttt{PAIR} achieves the strongest cross-run reliability
among compressed methods in every main benchmark--scope combination,
consistently exceeding the competing prompt-adaptation baseline.
Without modifying the downstream agent, \texttt{PAIR} brings compressed execution close to the no-compression baseline and sometimes numerically
exceeds it.

\end{abstract}

\section{Introduction}
Advances in language models and agent frameworks have enabled LLM agents to tackle increasingly complex, long-horizon tasks~\citep{yao2023react,shinn2023reflexion,wang2024codeact}.
These agents navigate the web~\citep{zhou2024webarena}, control software applications~\citep{trivedi2024appworld}, and carry out office workflows~\citep{wang2024officebench}. Completing such 
tasks requires many rounds of tool use and interaction with the environment. Actions, observations, and intermediate results accumulate throughout execution, causing the agent's context to grow continually.
Since context windows are finite and inference cost grows with input length, retaining the full interaction history eventually becomes impractical. Context compression is therefore central to the design of long-horizon agents.

Prior context-compression methods primarily target static contexts, such as long prompts
and retrieved documents, whereas long-horizon agents require repeated compression as their
interaction histories grow during execution. 
General prompt-compression methods use token selection or rewriting to preserve semantic content or downstream utility under a reduced context budget~\citep{li2023selective,jiang2023llmlingua,jiang2024longllmlingua,pan2024llmlingua,shandilya2025taco}, but do not directly account for how compressed context affects subsequent agent execution. 
Recent work has therefore developed compression methods specifically for agent trajectories. ReSum periodically summarizes interaction histories and optionally adapts the agent policy through ReSum-GRPO, while SUPO jointly optimizes summarization and tool use through reinforcement learning \citep{wu2025resum,lu2025supo}.
These approaches can adapt the agent to operate with compressed histories, but require modifying the downstream agent. In many deployed settings, however, the agent is fixed, making the compression mechanism itself the primary target for improvement. Systems such as OpenClaw and Hermes Agent address this setting using structured prompts that specify what information should be preserved during compression
\citep{openclaw_compaction,hermes_context_compression}. 
Building on this prompt-based design, ACON adapts the compression prompt using feedback from successful full-context and failed compressed trajectories \citep{kang2025acon}. This provides a practical way to improve compression without modifying the agent, but relies on trajectory-level outcomes to determine how the compression prompt should be revised.

However, trajectory-level feedback is too coarse for recurrent context compression. A long-horizon trajectory may contain several compression events, so its final outcome cannot identify which one altered later execution. Agent stochasticity further confounds attribution, as independent rollouts from the same context may diverge. Thus, comparing full-context and compressed trajectories does not reliably isolate any compression's effect. We call each replacement of accumulated history by a compressed representation a \emph{compression boundary} (e.g., summarizing the first five interactions before the sixth action). We therefore ask three progressively focused questions: \textit{(i) How does recurrent compression affect execution reliability and efficiency? (ii) Is degradation localized to individual boundaries or accumulated across repeated compressions? (iii) Can boundary-level evidence improve the compression prompt for target tasks?}

We introduce \texttt{PAIR} (\textbf{\underline{P}}rompt \textbf{\underline{A}}daptation using \textbf{\underline{I}}nterventional \textbf{\underline{R}}ollouts), a framework for adapting context-compression prompts for frozen long-horizon agents. Instead of attributing trajectory-level performance differences directly to compression, \texttt{PAIR} uses counterfactual continuations to estimate the effect of individual compression events. By comparing repeated rollouts from the same environment state before and after compression, it identifies compression boundaries that reliably alter task success or execution efficiency despite agent stochasticity. The resulting boundary-level evidence is then used to diagnose recurring compression errors and revise the compression prompt. Throughout this process, the agent, compressor model, tools, and decoding configuration remain fixed. 
Our contributions are threefold:
\begin{itemize}[leftmargin=0.25cm, nosep]
    \item \textbf{Recurrent Compression Analysis:} We characterize how
    recurrent compression affects long-horizon execution, showing that
    it degrades cross-run reliability before solvability and introduces
    additional recovery steps. 
    Across environments, many compression events modestly increase the number of subsequent interaction steps, while substantial drops in task success or large increases in execution length are concentrated at a small number of compression boundaries.
    \item \textbf{Boundary-Level Attribution \& Adaptation:} 
    We use paired counterfactual continuations from the same environment state, with versus without each compression, to estimate how a compression changes subsequent task success and execution length. We establish a performance-difference identity connecting these boundary-level effects to the downstream compressor objective. Using this evidence, \texttt{PAIR} identifies harmful compressions, diagnoses the information they fail to preserve, and adapts the corresponding sections of a fixed, deployment-compatible compression template.
    \item \textbf{Empirical Effectiveness:} Experiments on
AppWorld~\citep{trivedi2024appworld},
OfficeBench~\citep{wang2024officebench}, and
$\tau^2$-Bench Retail~\citep{barres2025tau} show that
\texttt{PAIR} achieves the highest rate of consistent task completion
across repeated runs under both history-only and prefix-conditioned
compression. It consistently exceeds the strongest competing
prompt-adaptation baseline while approaching or occasionally
numerically exceeding uncompressed execution.
\end{itemize}

\section{Preliminary}\label{sec:prelim}

\noindent\textbf{Notations.}
We study context compression for a fixed downstream agent $\mathcal{M}$ interacting with an environment $\mathcal{E}$. The agent's language model, system prompt, action interface, output parser, and decoding procedure remain unchanged. These fixed components are absorbed into $\mathcal{M}$, while $z_t$ denotes the mutable textual context visible to the agent.
Let $u\sim\mathcal{D}$ be a task instruction. A rollout is $\tau=(u,a_1,o_1,\ldots,a_T,o_T)$, and the complete history before step $t$ is $h_t=(u,a_1,o_1,\ldots,a_{t-1},o_{t-1})$. Given $z_t$, the agent samples $a_t\sim\mathcal{M}(\cdot\mid z_t)$ and receives an observation $o_t\sim\mathcal{E}(\cdot\mid h_t,a_t)$. Full-context execution uses $z_t=h_t$. We use $|x|$ to denote the token length of a textual context $x$. Since $|h_t|$ grows with the interaction horizon, it may eventually exceed the available context budget.
All notations and symbols are summarized in Table~\ref{tab:notation}.

\subsection{Problem Definition}
\label{subsec:problem_def}

To keep the agent-visible context within budget $B$, a compressor $\mathcal{C}$ replaces an overlength context with a bounded textual representation~\citep{kang2025acon,lu2025supo,AG2_2024,hermes_context_compression,openclaw_compaction}. After action $a_t$ and observation $o_t$, define the pre-compression context as $\bar{z}_{t+1}=z_t\oplus(a_t,o_t)$, where $\oplus$ denotes textual concatenation. The next context is recursively updated as
\begin{equation}
z_{t+1}=
\begin{cases}
\bar{z}_{t+1},
& |\bar{z}_{t+1}|\leq B,\\
c_{t+1},\quad
c_{t+1}\sim\mathcal{C}(\cdot\mid\bar{z}_{t+1},B),
& |\bar{z}_{t+1}|>B.
\end{cases}
\label{eq:context_transition}
\end{equation}
Here, $\mathcal{C}(\cdot\mid x,B)$ has support only on contexts with $|c|\leq B$. After compression, the replacement context is the only historical record for the frozen agent and subsequent compression steps.

Let $p_{\mathcal{C},B}(\tau\mid u)$ denote the rollout distribution induced by the frozen agent, environment, compressor, and recursive context transition, and let $R(\tau)$ denote the environment-defined terminal reward.

\begin{problem}[Compressor Adaptation for Frozen Agents]
\label{def:problem}
Given a task distribution $\mathcal{D}$, a frozen downstream agent $\mathcal{M}$, an environment $\mathcal{E}$, a context budget $B$, and an admissible compressor class $\mathfrak{C}$, adapt the compressor to the frozen agent by solving
\begin{equation}
\mathcal{C}^{*}\in
\operatorname*{arg\,max}_{\mathcal{C}\in\mathfrak{C}}
\mathbb{E}_{\substack{
u\sim\mathcal{D},\,
\tau\sim p_{\mathcal{C},B}(\cdot\mid u)
}}
\left[R(\tau)\right].
\label{eq:context_compression_problem}
\end{equation}
\end{problem}

Problem~\ref{def:problem} evaluates compression by its effect on downstream task performance rather than by its fidelity in reconstructing the original history. We implement $\mathcal{C}$ using a frozen autoregressive language model conditioned on a natural-language compression prompt $P$, which induces the policy $\pi_P(c\mid x,B)$. Because compressed contexts may preserve different information and induce different downstream returns, we adapt $P$ according to its effect on subsequent agent behavior.

% \newpage
\section{Diagnosing Compression-Induced Instability}\label{sec:motivation}
In this section, we study how recurrent context compression affects
long-horizon agent execution. We first show that compression primarily
reduces cross-run reliability and increases interaction steps, even
when tasks remain solvable (Section~\ref{sec:repeated-compression}).
We then introduce a counterfactual protocol that attributes these
effects to individual compression boundaries
(Section~\ref{sec:counterfactual-evaluation}). The resulting analysis
shows that compression commonly introduces mild execution burden,
whereas severe degradation is concentrated at a small number of
boundaries (Section~\ref{sec:hazard-distribution}).

\subsection{Compression Degrades Reliability Before Solvability}
\label{sec:repeated-compression}

\begin{figure*}[t]
    \centering
    \captionsetup[subfigure]{font=small,skip=1pt}

    \makebox[\linewidth][c]{%
        \begin{subfigure}[t]{0.27\linewidth}
            \centering
            \includegraphics[width=\linewidth]
            {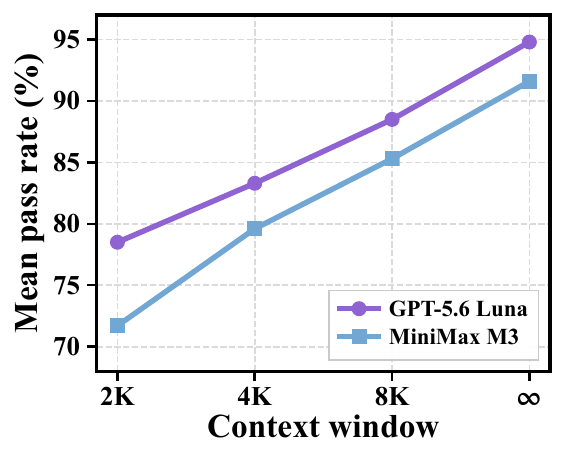}
            \caption{Compression performance.}
            \label{fig:compression-performance}
        \end{subfigure}
        \hspace{0.015\linewidth}
        \begin{subfigure}[t]{0.27\linewidth}
            \centering
            \includegraphics[width=\linewidth]
            {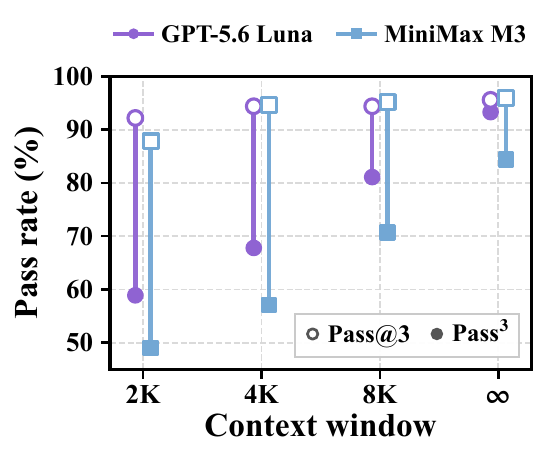}
            \caption{Cross-run reliability.}
            \label{fig:cross-run-reliability}
        \end{subfigure}
        \hspace{0.015\linewidth}
        \begin{subfigure}[t]{0.27\linewidth}
            \centering
            \includegraphics[width=\linewidth]
            {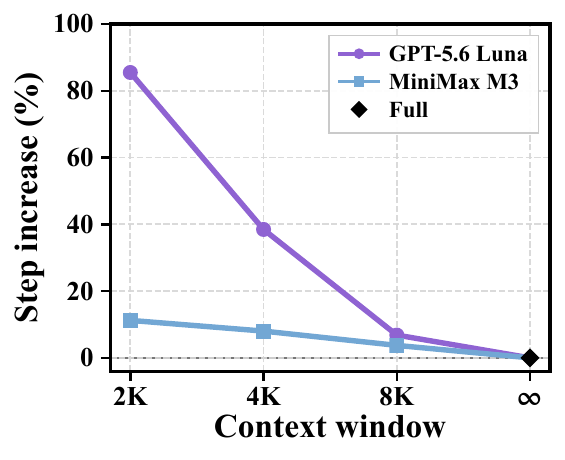}
            \caption{Execution overhead.}
            \label{fig:execution-overhead}
        \end{subfigure}%
    }

    \vspace{-1mm}
    \caption{
        \textbf{Recurrent context compression degrades reliability
        before solvability.}
        (a) Mean pass rate declines as the compression window shrinks.
        (b) Pass$^3$ falls faster than Pass@3, indicating greater
        cross-run variability; lines connect the two metrics.
        (c) Smaller windows increase interaction steps relative to
        full-context execution.
    }
    \label{fig:compression-effects}
    \vspace{-4mm}
\end{figure*}

Equation~\eqref{eq:context_transition} makes compression a recurrent
intervention: each compressed context guides subsequent actions and
becomes part of the input to later compressions. We examine its aggregate
effect on 90 AppWorld training tasks~\citep{trivedi2024appworld}. For
each context window, we run GPT-5.6 Luna~\citep{openai2026gpt56luna}
and MiniMax-M3~\citep{minimax2026m3} three times under full-context
execution and OpenClaw-style recurrent compression~\citep{openclaw_compaction}.
Additional experimental details are provided in
Appendix~\ref{app:experimental_details}.

Figure~\ref{fig:compression-effects} shows that smaller compression
windows reduce mean pass rate, but affect reliability considerably more
than solvability. Pass$^3$, the fraction of tasks solved in all three
runs~\citep{barres2025tau}, declines much faster than Pass@3, the
fraction solved at least once. Compressed runs also require more
interaction steps. Compression therefore first turns consistently
solved tasks into intermittently solved and less efficient ones, rather
than making them uniformly unsolvable.

AppWorld maintains persistent application state, so information removed
from the active context can often be recovered through further
interaction. This reflects a practical deployment setting, as agent
systems such as OpenClaw may archive pre-compression histories in memory
or local files for later retrieval. The additional steps in
Figure~\ref{fig:compression-effects} suggest that agents often attempt
to reconstruct missing execution state after compression. Some
rollouts recover and complete the task, while others fail.

These observations identify two harmful effects of compression: reduced
continuation success and increased interaction burden. Their values do
not reveal whether degradation accumulates across many compressions or
is concentrated at boundaries. We distinguish these possibilities next.

\subsection{Counterfactual Evaluation at Compression Boundaries}
\label{sec:counterfactual-evaluation}

A terminal outcome cannot identify the contribution of an individual
compression. Multiple context replacements may occur before termination,
while independent rollouts may diverge even from the same context. We
therefore evaluate each realized compression at the boundary where it is
introduced.

Consider a replacement
$\bar z_{t+1}\rightarrow c_{t+1}$ in
Equation~\eqref{eq:context_transition}, where
$c_{t+1}\sim\mathcal C(\cdot\mid\bar z_{t+1},B)$, and let
$\xi_t=(h_{t+1},\bar z_{t+1})$ denote the corresponding boundary state.
We restore the environment state reached after $h_{t+1}$ and continue
execution from either the pre-compression context $\bar z_{t+1}$ or the
post-compression context $c_{t+1}$. Both conditions use the same frozen
agent and tools, with compression disabled thereafter, and are rolled out
to termination. This isolates the current context replacement from later
compression events.

Let $\varnothing$ denote execution without further compression. For any
context $x$, let $Q_{\varnothing}^{R}(\xi_t,x)$ denote the expected
terminal reward obtained by continuing from $z_{t+1}=x$, and let
$Q_{\varnothing}^{L}(\xi_t,x)$ denote the expected number of remaining
interaction steps.

\begin{definition}[Boundary-level compression effect]
\label{def:boundary-compression-effect}
The effect of summary $c_{t+1}$ at boundary state $\xi_t$ is
\begin{equation}
\mathbf D_{\varnothing}(\xi_t,c_{t+1})
=
\begin{pmatrix}
H_{\varnothing}(\xi_t,c_{t+1})\\
B_{\varnothing}(\xi_t,c_{t+1})
\end{pmatrix}
=
\begin{pmatrix}
Q_{\varnothing}^{R}(\xi_t,\bar z_{t+1})
-
Q_{\varnothing}^{R}(\xi_t,c_{t+1})
\\[2pt]
Q_{\varnothing}^{L}(\xi_t,c_{t+1})
-
Q_{\varnothing}^{L}(\xi_t,\bar z_{t+1})
\end{pmatrix}.
\label{eq:boundary-compression-effect}
\end{equation}
We call $H_{\varnothing}$ the \emph{outcome hazard} and
$B_{\varnothing}$ the \emph{interaction burden}. Positive values indicate
that compression reduces expected terminal reward or increases expected
continuation length, respectively.
\end{definition}

The two components capture complementary consequences of the same
intervention. Outcome hazard measures the resulting change in task
completion under the fixed execution horizon. Interaction burden measures
the additional execution required to recover from the compressed context.
Because agent rollouts are subject to a step limit, this burden also
reduces the remaining execution budget and may eventually turn an
otherwise recoverable disruption into a timeout failure.

Given $m$ independent continuations from each context, we estimate
$Q_{\varnothing}^{R}(\xi_t,x)$ by
$m^{-1}\sum_{i=1}^{m}R(\tau_{t,i}^{x})$ and
$Q_{\varnothing}^{L}(\xi_t,x)$ by
$m^{-1}\sum_{i=1}^{m}\ell(\tau_{t,i}^{x})$, where
$x\in\{\bar z_{t+1},c_{t+1}\}$.
Substituting these sample means into
Equation~\eqref{eq:boundary-compression-effect} yields
$\widehat{\mathbf D}_{\varnothing}
=(\widehat H_{\varnothing},\widehat B_{\varnothing})$.
Restoring the same environment state controls for prior execution, while
disabling subsequent compression isolates the effect of the current
replacement; repeated continuations reduce stochastic rollout noise.
The outcome hazard is directly aligned with the compressor objective in
Problem~\ref{def:problem}. Denote this objective by
$
\mathcal J(\mathcal C)
=
\mathbb E_{u\sim\mathcal D,\,
\tau\sim p_{\mathcal C,B}(\cdot\mid u)}
[R(\tau)].
$

\begin{theorem}[Boundary performance-difference identity]
\label{thm:boundary-performance-difference}
Let $\varnothing$ denote execution without compression. For any compressor
$\mathcal C$,
\begin{equation}
\mathcal J(\mathcal C)-\mathcal J(\varnothing)
=
-\mathbb E_{\substack{
u\sim\mathcal D,\;
\tau\sim p_{\mathcal C,B}(\cdot\mid u)
}}
\!\left[
\sum_{t\in\mathcal B(\tau)}
H_{\varnothing}(\xi_t,c_{t+1})
\right].
\label{eq:boundary-performance-difference}
\end{equation}
Thus, minimizing expected cumulative boundary hazard is equivalent to
maximizing the objective in Problem~\ref{def:problem}.
\end{theorem}

Full proof is provided in
Appendix~\ref{app:boundary-effect-proof}.
Equation~\eqref{eq:boundary-performance-difference} shows that the gap to
no-compression execution is exactly the expected cumulative outcome hazard
over compression boundaries. Since $\mathcal J(\varnothing)$ is independent
of $\mathcal C$, reducing boundary-level outcome hazard directly improves
the objective in Problem~\ref{def:problem}.
Interaction burden does not appear separately in the objective in
Problem~\ref{def:problem}, since timeout failures are already reflected in
terminal reward. It nevertheless provides a finer diagnostic of how
compression approaches such failures: additional recovery steps consume
the remaining execution budget even when the continuation still succeeds.

\subsection{Compression Burden Is Common, but Severe Effects Are Sparse}
\label{sec:hazard-distribution}

\begin{figure}[!t]
    \centering
    \captionsetup[subfigure]{font=small,skip=1pt}

    \makebox[\linewidth][c]{%
        \begin{subfigure}[t]{0.255\linewidth}
            \centering
            \includegraphics[width=\linewidth]
            {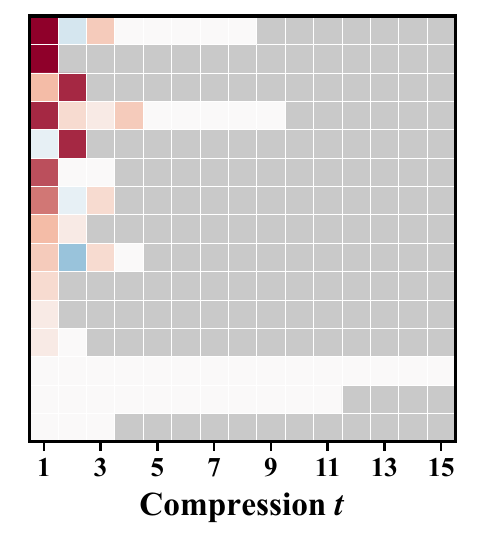}
            \caption{Failed.}
            \label{fig:hazard-failed}
        \end{subfigure}
        \hspace{0.012\linewidth}
        \begin{subfigure}[t]{0.185\linewidth}
            \centering
            \includegraphics[width=\linewidth]
            {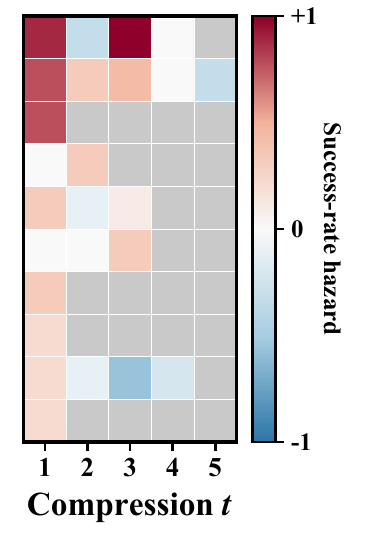}
            \caption{Successful.}
            \label{fig:hazard-successful}
        \end{subfigure}
        \hspace{0.012\linewidth}
        \begin{subfigure}[t]{0.305\linewidth}
            \centering
            \includegraphics[width=\linewidth]
            {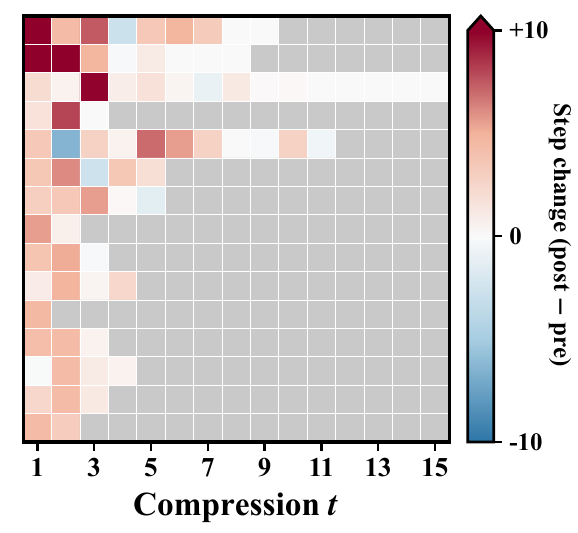}
            \caption{Interaction burden.}
            \label{fig:boundary-step-change}
        \end{subfigure}%
    }

    \vspace{-2mm}
    \caption{
        \textbf{Boundary-level effects of recurrent compression.}
        Panels (a) and (b) show outcome hazard in failed and successful
        trajectories; panel (c) shows interaction burden. Red indicates
        degradation or additional steps, blue indicates improvement, and
        gray denotes absent boundaries.
    }
    \label{fig:hazard-distribution}
    \vspace{-5mm}
\end{figure}

We apply the counterfactual protocol with $m=9$ continuations per
condition to 197 compression boundaries from 82 trajectories: 67
boundaries from 15 failed trajectories and 130 from 67 successful ones.
Outcome degradation is sparse and not determined by terminal outcome.
Among failed trajectories, 43 boundaries leave continuation success
unchanged, 20 reduce it, and four improve it. Twelve of the 15
trajectories contain a positive-hazard boundary, but 11 contain only one
or two. Conversely, 27 boundaries in successful trajectories reduce
continuation success across 22 trajectories. Thus,
Panels~\ref{fig:hazard-failed} and~\ref{fig:hazard-successful} show that
failed trajectories contain mostly neutral compressions, while harmful
compressions can still precede successful execution. Interaction burden
is more common but usually moderate: as shown in
Panel~\ref{fig:boundary-step-change}, 151 of 197 boundaries increase
continuation length, but only 12 add more than five steps and three add
more than ten. Severe outcome degradation is also rare, with only 11
boundaries reducing continuation success by more than $0.5$. Overall,
recovery costs are common, while severe effects on success or execution
length remain localized. OfficeBench and $\tau^2$-Bench show the same
qualitative pattern; additional results appear in
Appendix~\ref{app:additional-boundary-effects}.

This structure limits trajectory-level attribution. ACON contrasts
successful full-context and failed compressed trajectories and revises
the compression guideline from traces~\citep{kang2025acon}.
When several compressions precede termination, the replacement
must be inferred retrospectively~\citep{zhang2025which}.
Our results show why this is difficult: most compressions preceding
failure are neutral, while harmful compressions also occur in successful
trajectories. This motivates localizing boundaries that reduce
continuation success or increase recovery burden, then using their
paired contexts and counterfactual continuations as targeted evidence
for improving the compression policy.

\section{Counterfactual Continuation-Guided Prompt Adaptation}
\label{sec:method}

Figure~\ref{fig:pair-framework} summarizes \texttt{PAIR}, which uses
paired counterfactual continuations to localize harmful compressions,
revise the relevant sections of a structured prompt, and select the
adapted prompt through end-to-end validation. Theorem~\ref{thm:boundary-performance-difference}
connects outcome hazard to the compressor objective. All model
parameters remain fixed; only the prompt is adapted.

\noindent\textbf{Step 1: Collecting compressed trajectories.}
Let $P_0$ denote the original structured compression prompt and
$\mathcal C_{P_0}$ its induced compressor. For each task in
$\mathcal D_{\mathrm{train}}$, we collect one trajectory using
$\mathcal C_{P_0}$ under the target context budget and retain every
realized compression boundary, regardless of terminal outcome.
This is necessary because successful trajectories may still contain
outcome-degrading boundaries, while interaction burden occurs in both
successful and failed executions.

\noindent\textbf{Step 2: Verifying harmful compression boundaries.}
For every boundary, we apply the counterfactual protocol from
Section~\ref{sec:counterfactual-evaluation} using three-round
successive halving. We omit the subscript $\varnothing$ from empirical
boundary-effect estimates in this section. Each boundary first receives
one PRE/POST continuation pair. After round $r$, active boundaries are
ranked by the normalized harm score
\begin{equation}
s_t^{(r)}
=
\max\!\left\{
\frac{\widehat H_t^{(r)}}{\tau_H},
\frac{\widehat B_t^{(r)}}{\tau_B}
\right\},
\label{eq:halving-score}
\end{equation}
where $\widehat H_t^{(r)}$ and $\widehat B_t^{(r)}$ use all $r$ pairs
observed so far. After each of the first two rounds, only the top half
receive another pair, for at most three pairs per boundary. Using all
continuations allocated to each surviving boundary, we retain those
satisfying $\widehat H_t\geq\tau_H$ or
$\widehat B_t\geq\tau_B$. The two criteria capture reduced continuation
success and increased recovery cost, respectively. Each retained
boundary provides a localized contrast consisting of the pre-compression
context, generated summary, and counterfactual continuation traces.

\noindent\textbf{Step 3: Adapting the compression prompt.}
For each retained boundary, an optimizer LLM receives the pre- and
post-compression contexts together with their counterfactual
continuations. It diagnoses how the summary altered subsequent
execution, such as by omitting necessary state, misrepresenting
completed progress, or inducing redundant recovery actions. These
localized diagnoses are then aggregated to revise the original prompt
$P_0$.

Unlike ACON, which allows the compression guideline to be rewritten as
a whole~\citep{kang2025acon}, \texttt{PAIR} keeps the template's section
structure, required fields, compression scope, and output format fixed,
and revises only the guidance within each section. This preserves
compatibility with the surrounding context-management pipeline while
allowing the adapted prompt to replace the original template directly.
We generate five candidate prompts
$\mathcal P_{\mathrm{cand}}=\{P^{(1)},\ldots,P^{(5)}\}$, each shared
across the target task family rather than specialized to an individual
trajectory.

\noindent\textbf{Step 4: Selecting the adapted prompt.}
Boundary evidence is collected under the original compressor, whereas
a revised prompt may change both summaries and the boundaries visited
during execution. Local improvements therefore require end-to-end
validation. We evaluate each candidate on the subset of training tasks
whose baseline trajectories contain the most compression events and
select the one with the highest pass rate, breaking ties by the lowest
mean number of interaction steps. This criterion prioritizes task
completion while using execution length to distinguish equally
successful prompts and preserve execution budget against timeout.
The selected prompt $P^\star$ is used without further adaptation for
held-out evaluation.

\begin{figure}[t]
    \centering
    \includegraphics[width=\linewidth]
    {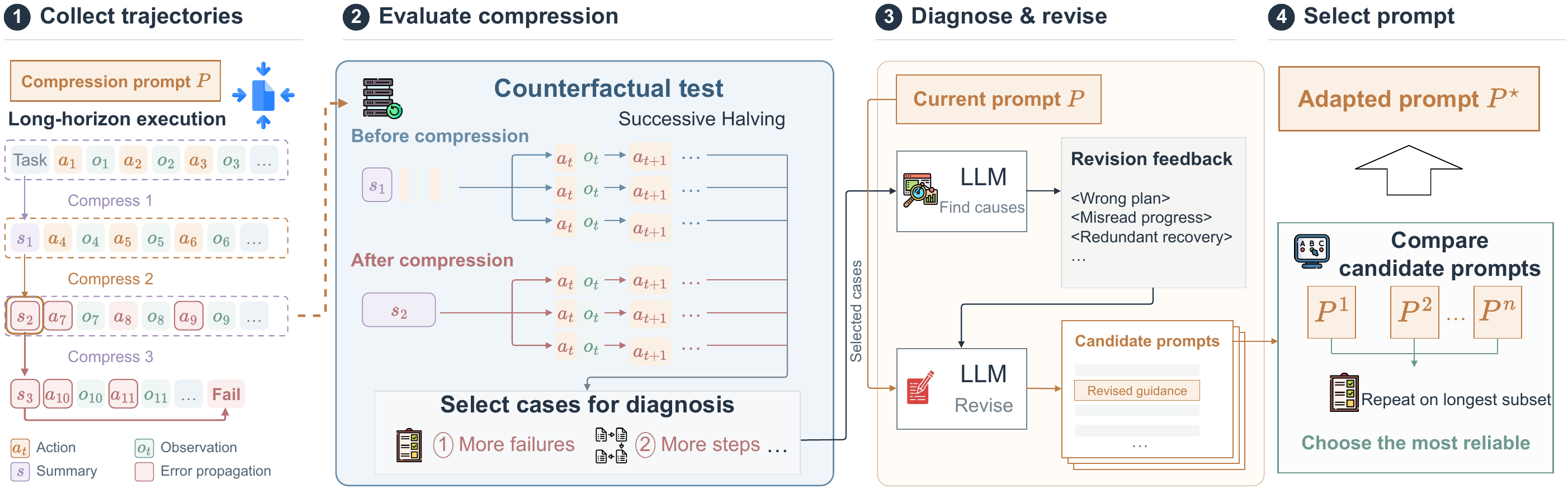}
    \caption{
        \textbf{Overview of \texttt{PAIR}.}
        Starting from trajectories generated with an initial structured
        compression prompt, \texttt{PAIR} uses paired counterfactual
        continuations to localize harmful compression boundaries.
        An optimizer diagnoses their downstream effects and revises the
        relevant prompt sections; candidate prompts are then evaluated
        end-to-end to select the adapted compression policy.
    }
    \label{fig:pair-framework}
    \vspace{-4mm}
\end{figure}

% \newpage
\section{Experiments}

\subsection{Experimental Setup}
\label{subsec:setup}

\noindent\textbf{Benchmarks.}
We evaluate primarily on two long-horizon agent benchmarks.
\textbf{AppWorld} contains stateful API-use tasks across simulated
applications and users~\citep{trivedi2024appworld}; we adapt prompts on
its training split and evaluate on the 168-task test-normal split.
\textbf{OfficeBench} covers multi-application office workflows
involving documents, spreadsheets, and email
\citep{wang2024officebench}. We additionally include
\textbf{$\tau^2$-Bench Retail} as a stress test in a structurally
different environment, where the agent must satisfy policy constraints
while interacting with both a simulated user and backend tools
\citep{barres2025tau}. Further details are provided in
Appendix~\ref{app:experimental_details}.

\noindent\textbf{Metrics.}
We evaluate each method over three independent runs.
\textbf{Acc.} is the mean task success rate, while
\textbf{Pass$^3$} is the fraction of tasks solved in all three runs and
measures cross-run reliability. We also report interaction
\textbf{Steps} and peak input length \textbf{Peak}, measured in
thousands of tokens. Our efficiency analysis reports
cumulative tokens, including all agent and compressor inputs and
outputs.

\noindent\textbf{Baselines.}
We use \textbf{No Compression}, which retains full history, as
the uncompressed reference and compare \texttt{PAIR} with four
compression baselines: \textbf{(1) FIFO} retains recent history
within the context budget; \textbf{(2) LLMLingua-2} removes
tokens using a learned classifier~\citep{pan2024llmlingua};
\textbf{(3) Prompting} uses the original OpenClaw structured
compression prompt~\citep{openclaw_compaction}; and
\textbf{(4) ACON}~\citep{kang2025acon} adapts compression guidelines
from full-context and compressed trajectories. We evaluate its
utility-oriented \texttt{ACON-UT} and compression-aware
\texttt{ACON-UTCO} variants. ACON and \texttt{PAIR} use the same
initial prompt, adaptation tasks, optimizer model, context budget,
candidate count, and end-to-end selection budget, differing only in
their adaptation procedures.

\noindent\textbf{Compression Scopes.}
We evaluate prompting-based methods under two scopes.
\textbf{History-only} compression summarizes the interaction history
while retaining the fixed prefix separately, as in OpenClaw and Hermes
Agent~\citep{openclaw_compaction,hermes_context_compression}.
\textbf{Prefix-conditioned} compression additionally exposes the fixed
prefix to the compressor, while preserving it unchanged in the
downstream context to enable KV-cache reuse, as in DeepSeek Harness and
Codex~\citep{deepseek_harness_compaction,openai_codex_compaction}.
All prompting-based methods use the same compressor, context budget,
and execution configuration.

\begin{table*}[t]
\centering
\caption{
Results on the \textbf{AppWorld} test-normal split by difficulty,
aggregated over three independent runs. Acc. is mean task success,
with across-run standard deviation shown as a subscript; Pass$^3$ is
the fraction of tasks solved in all three runs; Steps and Peak are
mean interaction steps and peak input length ($10^3$ tokens).
\textit{No compression} is the uncompressed reference. History-only
and prefix-conditioned compression differ in whether the compressor
receives the fixed prefix, which remains separately visible to the
agent in both settings. Bold denotes the best compressed result within
each block; gray rows indicate \texttt{PAIR}.
}
\vspace{-0.1in}
\label{tab:appworld_gpt_luna}
\resizebox{\textwidth}{!}{
\begin{tabular}{lccccccccccccc}
\toprule
\multirow{2}{*}{Method} & 
\multicolumn{4}{c}{Average (168)} & 
\multicolumn{3}{c}{Easy (57)} & 
\multicolumn{3}{c}{Medium (48)} & 
\multicolumn{3}{c}{Hard (63)} \\
\cmidrule(lr){2-5} \cmidrule(lr){6-8} \cmidrule(lr){9-11} \cmidrule(lr){12-14}
& Acc. $\uparrow$ & Pass$^3$ $\uparrow$ & Steps $\downarrow$ &Peak $\downarrow$
& Acc. $\uparrow$ & Pass$^3$ $\uparrow$ & Steps $\downarrow$
& Acc. $\uparrow$ & Pass$^3$ $\uparrow$ & Steps $\downarrow$
& Acc. $\uparrow$ & Pass$^3$ $\uparrow$ & Steps $\downarrow$\\
\midrule
\rowcolor{gray!30}\multicolumn{14}{c}{\textbf{Agent:} \texttt{GPT-5.6 Luna} / \textbf{Compressor:} \texttt{GPT-5.6 Luna}}\\
\midrule
\textit{No compression} & \textit{84.1}\std{0.9} & \textit{78.6} & \textit{12.1} & \textit{12.53}& \textit{97.7}\std{1.0} & \textit{93.0} & \textit{9.6} & \textit{86.8}\std{2.4} & \textit{81.2} & \textit{12.2} & \textit{69.8}\std{0.0} & \textit{63.5} & \textit{14.4} \\
\midrule[0.1pt]
\multicolumn{14}{l}{\textbf{History-Only Compression}}\\
FIFO & 53.2\std{1.2} & 45.2 & 30.0 & \textbf{7.14} & 95.9\std{2.0} & 93.0  & 9.6 & 47.9\std{4.2} & 37.5 & 33.7 & 18.5\std{3.3} & 9.5 & 42.8 \\
LLMLingua-2 &  {71.4}\std{2.3} &  {61.3} & 19.2 & 10.53  &97.7\std{2.7} &   94.7  &   \textbf{8.7} & 77.8\std{2.4} &   66.7 & 17.6 & 42.9\std{6.3} & 27.0 & 30.0 \\
Prompting & 78.6\std{2.7} & 68.5 &   17.0 &  9.23 & 94.2\std{2.0} & 89.5  & 10.0 & 81.2\std{2.1} & 66.7 &   15.5 &  62.4\std{6.4} &  50.8 &  24.4 \\
ACON-UT& 81.2\std{1.2} & 70.8 & 18.9 & 9.36 &94.2\std{4.4} & 82.5 & 10.4 &  83.3\std{5.5} & 70.8 & 18.5 & 67.7\std{4.0} &  \textbf{60.3} & 26.8 \\
ACON-UTCO & 77.0\std{2.3} & 66.1 & 19.2 & 9.24 &95.9\std{2.0} & 91.2 & 10.7 & 78.5\std{2.4} & 66.7 & 18.4 & 58.7\std{4.2} & 42.9 &  27.4  \\
\rowcolor{gray!15}\texttt{PAIR} (Ours) & \textbf{84.7}\std{1.7} & \textbf{79.8} & \textbf{16.6} & 9.03 &\textbf{98.2}\std{1.8} & \textbf{96.5} & 10.1 & \textbf{89.6}\std{2.1} & \textbf{85.4} & \textbf{15.3} & \textbf{68.8}\std{1.8} & \textbf{60.3} &  \textbf{23.3}  \\
\midrule
\multicolumn{14}{l}{\textbf{Prefix-Conditioned Compression}}\\
Prompting & 81.0\std{1.6} & 72.0 & 15.3 & 9.51&94.7\std{1.8} & 87.7  & \textbf{10.0} & 82.6\std{2.4} & 75.0 & 14.5 & 67.2\std{2.4} & 55.6 & 20.8 \\
ACON-UT & 81.3\std{0.9} & 72.6 & 17.3 & 10.21 &95.9\std{2.0} & 91.2 & 10.6 & 83.3\std{1.2} & 72.9 & 15.7 & 66.7\std{0.9} & 55.6 &  24.5  \\
ACON-UTCO & 78.2\std{2.5} & 67.3 & 17.2 & \textbf{8.69} &95.3\std{1.0} & 89.5 & 10.4 & 77.1\std{4.2} & 66.7 & 15.9 & 63.5\std{4.2} & 47.6 &  24.4  \\
\rowcolor{gray!15}\texttt{PAIR} (Ours) & \textbf{85.9}\std{2.4} & \textbf{80.4} & \textbf{15.2} & 9.47 & \textbf{97.7}\std{2.0} & \textbf{94.7} & 10.1 & \textbf{89.6}\std{0.0} & \textbf{85.4} & \textbf{14.1} & \textbf{72.5}\std{5.1} & \textbf{63.5} &  \textbf{20.7}  \\
\bottomrule
\end{tabular}
}
\vspace{-4mm}
\end{table*}

\begin{table*}[t]
\captionsetup[subtable]{font=small,skip=2pt}
\centering
\caption{
Results on \textbf{OfficeBench} and $\boldsymbol{\tau^2}$-\textbf{Bench Retail} domain,
aggregated over three runs. Acc. is mean task success, Pass$^3$ is
success across all three runs, Steps is mean interaction length, and
Peak is peak input length ($10^3$ tokens). History-only and
prefix-conditioned compression differ in whether the compressor
receives the fixed prefix, which remains visible to the agent in both
settings. Bold denotes the best result within each block; grey rows
indicate \texttt{PAIR}.
}
\vspace{-0.1in}
\label{tab:qa_officebench_main}

\begin{subtable}[t]{0.48\textwidth}
\centering
\caption{OfficeBench}
% \vspace{-0.08in}
\label{tab:4_officebench_main}
\resizebox{0.92\linewidth}{!}{%
\setlength{\tabcolsep}{6pt}
\begin{tabular}{lcccc}
\toprule
Method & Acc. $\uparrow$ & Pass$^3$ $\uparrow$  & Steps $\downarrow$ & Peak $\downarrow$\\
\midrule
\rowcolor{gray!30}
\multicolumn{5}{c}{\textbf{Agent:} \texttt{GPT-5.6 Luna} / \textbf{Compressor:} \texttt{GPT-5.6 Luna}} \\
\midrule
\textit{No Compression} & \textit{82.1}\std{4.8} & \textit{75.8}  & \textit{11.2}   & \textit{7.02}\\
\midrule
\multicolumn{5}{l}{\textbf{History-Only Compression}}\\
FIFO                    & 62.1\std{1.8}   & 54.7   & 23.2   & \textbf{3.48}   \\
LLMLingua-2           & 76.1\std{1.2}   & 68.4   & 14.4   & 4.14   \\ 
Prompting             & 76.5\std{2.4}   & 65.3   & \textbf{11.9}   & 4.00 \\
ACON-UT                 & 77.2\std{1.6}   & 68.4   & 13.3   & 4.20   \\
ACON-UTCO               & 73.0\std{5.8}   & 60.0   & 13.6   & 3.95   \\
\rowcolor{gray!15}
\texttt{PAIR} (Ours)         & \textbf{81.1}\std{2.8}   & \textbf{70.5}   & \textbf{11.9}   & 3.89 \\
\midrule
\multicolumn{5}{l}{\textbf{Prefix-Conditioned Compression}}\\
Prompting             & 78.9\std{2.8}   & 66.3   & 13.3   & 4.08 \\
ACON-UT                 & 75.1\std{0.6}   & 61.1   & 15.5  & 4.26   \\
ACON-UTCO               & 78.9\std{3.8}   & 65.3   & 13.6   & \textbf{4.01}   \\
\rowcolor{gray!15}
\texttt{PAIR} (Ours)         & \textbf{81.4}\std{0.6}   & \textbf{72.6}   & \textbf{11.8}   & 4.03 \\
\bottomrule
\end{tabular}%
}
\end{subtable}
\hfill
\begin{subtable}[t]{0.48\textwidth}
\centering
\caption{$\tau^2$-Bench Retail}
% \vspace{-0.08in}
\label{tab:3_qa_main}
\resizebox{0.92\linewidth}{!}{%
\setlength{\tabcolsep}{6pt}
\begin{tabular}{lcccc}
\toprule
Method & Acc. $\uparrow$ & Pass$^3$ $\uparrow$  & Steps $\downarrow$ & Peak $\downarrow$\\
\midrule
\rowcolor{gray!30}
\multicolumn{5}{c}{\textbf{Agent:} \texttt{GPT-5.6 Luna} / \textbf{Compressor:} \texttt{GPT-5.6 Luna}} \\
\midrule
\textit{No Compression} & 85.8\std{5.4}   & 72.5   & 12.6   & 6.79\\
\midrule
\multicolumn{5}{l}{\textbf{History-Only Compression}}\\
FIFO                    & 64.2\std{1.9}   & 45.0   & 17.7   & 5.43   \\
LLMLingua-2           & 75.0\std{4.4}   & 60.0   & \textbf{12.9}   & 5.96   \\ 
Prompting             & 74.2\std{8.2}   & 55.0   & 13.0   & 5.40 \\
ACON-UT                 & 73.3\std{6.3}   & 55.0   & 13.9   & 5.77   \\
ACON-UTCO               & \textbf{80.8}\std{3.8}   & 60.0   & 15.7   & \textbf{5.34}   \\
\rowcolor{gray!15}
\texttt{PAIR} (Ours)         & \textbf{80.8}\std{7.6}   & \textbf{62.5}   & 13.5   & 5.53 \\
\midrule
\multicolumn{5}{l}{\textbf{Prefix-Conditioned Compression}}\\
Prompting             & 75.0\std{5.0}   & 57.5   & 13.2   & 5.47 \\
ACON-UT                 & 73.3\std{9.5}   & 52.5   & \textbf{12.9}   & 5.89   \\
ACON-UTCO               & 73.3\std{3.8}   & 52.5   & 13.7   & \textbf{5.36}   \\
\rowcolor{gray!15}
\texttt{PAIR} (Ours)         & \textbf{77.5}\std{2.5}   & \textbf{60.0}   & 13.4   & 5.48 \\
\bottomrule
\end{tabular}%
}
\end{subtable}

\vspace{-0.2in}
\end{table*}

\subsection{Main Results}

Results are reported in
Tables~\ref{tab:appworld_gpt_luna}
and~\ref{tab:qa_officebench_main}. We evaluate prompting-based methods
under both history-only and prefix-conditioned compression.

\noindent\textbf{\texttt{PAIR} consistently improves reliability across
compression settings.}
\texttt{PAIR} achieves the highest Pass$^3$ among compressed methods
across all benchmarks and both compression scopes, outperforming the
best ACON variant by $2.1\%$--$10.0\%$. It also improves accuracy by
$1.8\%$--$3.9\%$ on AppWorld and OfficeBench while matching the best
ACON result on $\tau^2$-Bench Retail. Although neither history-only nor
prefix-conditioned compression dominates across all benchmarks,
\texttt{PAIR} improves reliability under both. On AppWorld, its largest
gains occur on medium and hard tasks, while easy-task performance is
already near saturation.

\noindent\textbf{Adapted compression can match uncompressed execution.}
On AppWorld, history-only \texttt{PAIR} reaches $84.7\%$ accuracy and
$79.8\%$ Pass$^3$, exceeding the uncompressed results of $84.1\%$ and
$78.6\%$. Prefix-conditioned \texttt{PAIR} also exceeds the
uncompressed Pass$^3$. On OfficeBench, both variants exceed $81\%$
accuracy, close to the uncompressed result of $82.1\%$, while using
substantially shorter contexts. Although a larger gap remains on
$\tau^2$-Bench Retail, prompt adaptation recovers much of the
performance lost to compression.

\noindent\textbf{The gains require neither longer contexts nor
substantial additional interaction.}
On OfficeBench, \texttt{PAIR} reduces peak input length from $7.02$K
tokens without compression to $3.89$K and $4.03$K while keeping the
number of steps close to the uncompressed reference. On AppWorld and
OfficeBench, it matches or reduces the steps required by Prompting. On
$\tau^2$-Bench Retail, it adds only $0.5$ and $0.2$ steps while
improving Pass$^3$ by $7.5\%$ and $5.0\%$, respectively.

\subsection{Cross-Model Transfer}
\label{sec:cross-agent-transfer}

\begin{wrapfigure}{r}{0.38\textwidth}
\vspace{-17mm}
\centering
\includegraphics[width=\linewidth]
{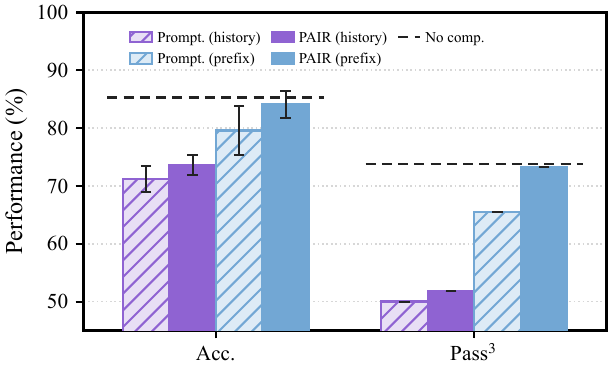}
\vspace{-6mm}
\caption{
\textbf{Zero-shot transfer to MiniMax-M3 on AppWorld.}
}
\label{fig:cross-agent-transfer}
\vspace{-3mm}
\end{wrapfigure}
We test whether the learned compression guidance transfers beyond the
model configuration used for adaptation. Prompts optimized with
GPT-5.6 Luna as both agent and compressor are applied without further
adaptation to a configuration in which both components use MiniMax-M3.
As shown in Figure~\ref{fig:cross-agent-transfer}, transferred
\texttt{PAIR} improves both accuracy and Pass$^3$ under both
compression scopes on AppWorld. The largest gain occurs under
prefix-conditioned compression, where accuracy increases from
$79.6\%$ to $84.1\%$ and Pass$^3$ from $65.5\%$ to $73.2\%$,
approaching the MiniMax-M3 no-compression results of $85.3\%$ and
$73.8\%$.
On OfficeBench, transfer improves accuracy under both scopes and
Pass$^3$ under history-only compression, while preserving Pass$^3$
under prefix-conditioned compression; full results are provided in
Appendix~\ref{app:cross-agent-transfer}. These results show that the
adapted guidance remains useful when both the downstream agent and
compressor are replaced, rather than depending only on the model
configuration used during adaptation.

\subsection{Deployment Efficiency}
\label{sec:efficiency}

\begin{figure}[t]
    \centering
    \includegraphics[width=0.97\linewidth]
    {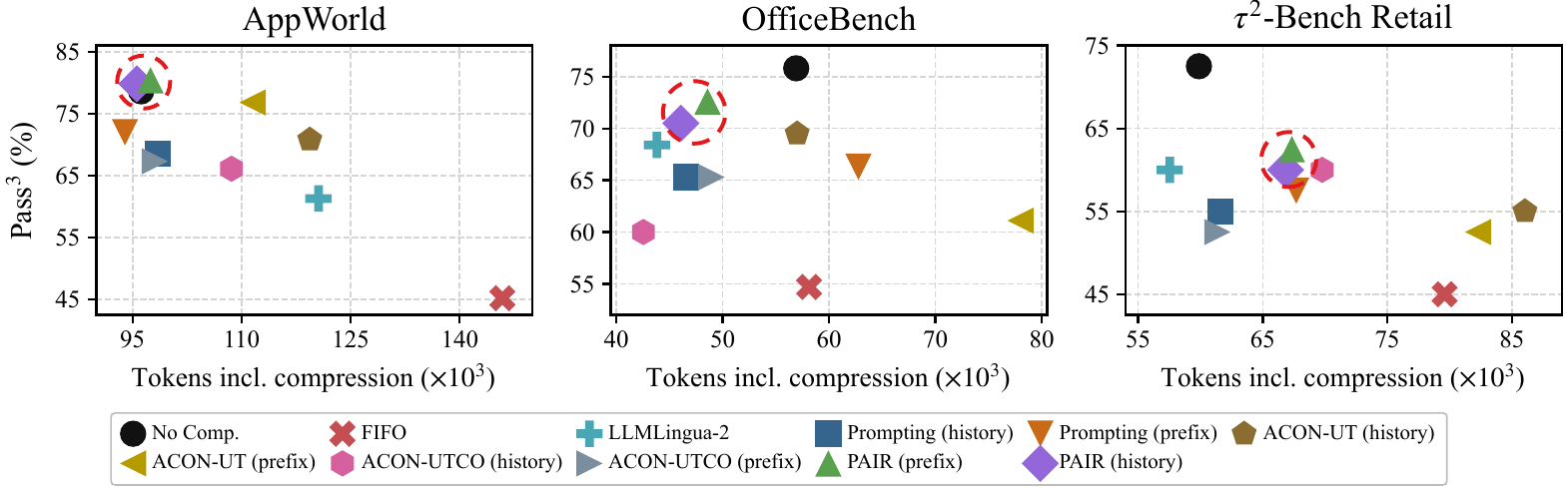}
    \caption{
\textbf{Reliability--efficiency after accounting for compression.}
Pass$^3$ versus average per-task token use, including agent and
compressor inputs and outputs. Higher and further left is better;
red circles mark \texttt{PAIR}. LLMLingua-2's local computation is
excluded.
}
    \label{fig:efficiency-tradeoff}
    \vspace{-4mm}
\end{figure}

\textbf{PAIR improves deployment reliability without systematically
increasing total token use.}
Figure~\ref{fig:efficiency-tradeoff} compares each \texttt{PAIR}
variant with the unadapted Prompting baseline under the same benchmark
and compression scope. \texttt{PAIR} improves Pass$^3$ in all six
comparisons by $5.0\%$--$11.3\%$, while using fewer tokens in four.
The remaining increases are $3.7\%$ on AppWorld and $8.5\%$ on
$\tau^2$-Bench. The largest efficiency gain occurs under
prefix-conditioned compression on OfficeBench, where Pass$^3$ rises
from $66.3\%$ to $72.6\%$ as token use falls from $62.8$K to $48.6$K.
The reliability gains therefore do not generally result from greater
inference expenditure.

\noindent\textbf{Compression does not always reduce total computation.}
Compression bounds the active context, but compressor overhead can
offset the resulting agent-side savings. After including this
overhead, \texttt{PAIR} matches or improves the best ACON Pass$^3$
under all six benchmark--scope combinations and uses fewer tokens in
five. The exception is prefix-conditioned compression on
$\tau^2$-Bench, where token use increases from $61.4$K to $67.3$K
while Pass$^3$ improves from $52.5\%$ to $60.0\%$. Relative to no
compression, \texttt{PAIR} uses essentially the same number of tokens
on AppWorld while attaining higher Pass$^3$, and uses $15\%$--$19\%$
fewer tokens on OfficeBench. On $\tau^2$-Bench, it uses approximately
$12\%$ more tokens and remains below the uncompressed reliability,
showing that a bounded active context does not guarantee lower
cumulative computation.

\section{Related Work}
\noindent\textbf{General-purpose context compression.}
Prior work compresses context through token selection, textual rewriting, or latent representations. Token-pruning methods include Selective Context and the LLMLingua family~\citep{li2023selective,jiang2023llmlingua,jiang2024longllmlingua,pan2024llmlingua}, while TACO--RL learns task-aware token selection from downstream rewards~\citep{shandilya2025taco}. RECOMP studies abstractive rewriting~\citep{xu2024recomp}, and AutoCompressor, Gist tokens, and ICAE encode long contexts into compact latent representations~\citep{chevalier2023adapting,mu2023gist,ge2024icae}. These methods generally compress a fixed input once, whereas long-horizon agents require recurrent compression whose outputs shape subsequent actions and later compressions.

\noindent\textbf{Context management for long-horizon agents.}
Systems such as OpenClaw and Hermes Agent periodically replace interaction histories with structured natural-language checkpoints~\citep{openclaw_compaction,openclaw_pruning,hermes_context_compression}. ReSum and SUPO adapt agents to compressed histories through post-training~\citep{wu2025resum,lu2025supo}, whereas ACON keeps the agent fixed and adapts compression guidelines from contrasts between full-context and compressed trajectories~\citep{kang2025acon}. We follow the fixed-agent setting but use repeated counterfactual continuations at individual compression boundaries to obtain localized evidence for adapting the compression prompt.

\section{Conclusion}
In this paper, we studied how recurrent context compression affects frozen
long-horizon agents. Using paired counterfactual continuations from the same
environment state, we attributed downstream behavioral changes to individual
compressions while reducing confounding from prior compression and rollout
variability. Our analysis shows that severe compression-induced degradation
is concentrated at a few harmful boundaries rather than accumulating uniformly
across the trajectory. Motivated by this localized structure, we introduced
\texttt{PAIR}, which uses harmful boundaries and their downstream consequences
as signals to diagnose compression errors and adapt a structured compression
prompt. Across benchmarks, \texttt{PAIR} improves task performance and
cross-run reliability while keeping the downstream agent fixed. These results
show that localized boundary-level evidence can guide context compression for
long-horizon agents.

%%%%%%%%%%%%%%%%%%%%%%%%%%%%%%%%%%%%%%%%%%%%%%%%%%%%%%%%%%%%

\bibliographystyle{iclr2026_conference}
\bibliography{main}

\newpage
\appendix
\section*{\Large{Appendix}}
\appendixtoc

\newpage

\section{Limitations and Future Work}

\texttt{PAIR} is designed for offline adaptation to a target task family.
Its counterfactual evaluation requires restorable environment states and
terminal continuations from both the pre- and post-compression contexts.
Successive halving reduces repeated sampling by concentrating additional
rollouts on boundaries with stronger evidence of harm, but every realized
boundary still receives at least one PRE/POST pair. Attribution cost
therefore grows with the number of compression boundaries and the remaining
execution horizon, making \texttt{PAIR} most directly applicable to
checkpointable or simulated environments.

Once the compression prompt is selected, \texttt{PAIR} requires no
additional model calls beyond standard deployment-time compression.
However, substantial shifts in tasks, tools, or execution environments may
require collecting new boundary-level evidence and readapting the prompt.
Developing cheaper boundary-effect estimators, extending adaptation to
irreversible environments, and incorporating boundary feedback online are
important directions for future work.

% \section{Extended Discussion on Related Work}

\section{Notations and Symbols}\label{app:notation}

We summarize the mathematical notations used throughout the paper in Table~\ref{tab:notation}.
\begin{center}
\small
\setlength{\LTcapwidth}{\linewidth}
\renewcommand{\arraystretch}{1.15}
\begin{longtable}{@{}l p{0.74\linewidth}@{}}
\caption{Summary of notations used in this paper.}
\label{tab:notation}\\
\toprule
\textbf{Symbol} & \textbf{Description} \\
\midrule
\endfirsthead

\toprule
\textbf{Symbol} & \textbf{Description} \\
\midrule
\endhead

\midrule
\multicolumn{2}{r}{\small Continued on next page} \\
\endfoot

\bottomrule
\endlastfoot

\multicolumn{2}{c}{\textit{Recurrent Context Compression}}\\

$\mathcal{M}$ &
Frozen downstream agent, including its language model, system prompt,
action interface, output parser, and decoding procedure \\

$\mathcal{E}$ &
Interactive environment \\

$u\sim\mathcal{D}$ &
Task instruction sampled from task distribution $\mathcal{D}$ \\

$\tau=(u,a_1,o_1,\ldots,a_T,o_T)$ &
Agent rollout terminating at step $T$ \\

$h_t$ &
Complete interaction history before step $t$,
$h_t=(u,a_1,o_1,\ldots,a_{t-1},o_{t-1})$ \\

$z_t$ &
Mutable textual context visible to the agent at step $t$ \\

$a_t$ &
Action sampled from the frozen agent,
$a_t\sim\mathcal{M}(\cdot\mid z_t)$ \\

$o_t$ &
Observation sampled from the environment,
$o_t\sim\mathcal{E}(\cdot\mid h_t,a_t)$ \\

$\oplus$ &
Textual concatenation operator \\

$\bar z_{t+1}$ &
Pre-compression context after appending the latest interaction,
$\bar z_{t+1}=z_t\oplus(a_t,o_t)$ \\

$B$ &
Maximum token budget for the agent-visible context \\

$|x|$ &
Token length of textual context $x$ \\

$\mathcal{C}$ &
Stochastic context compressor, with
$c\sim\mathcal{C}(\cdot\mid x,B)$ and $|c|\leq B$ \\

$c_{t+1}$ &
Replacement context generated from $\bar z_{t+1}$ at a compression
boundary \\

$\mathfrak{C}$ &
Admissible class of compressors \\

$p_{\mathcal{C},B}(\tau\mid u)$ &
Rollout distribution induced by the frozen agent, environment,
compressor $\mathcal{C}$, and context budget $B$ \\

$R(\tau)$ &
Environment-defined terminal reward of rollout $\tau$ \\

$\mathcal{J}(\mathcal{C})$ &
Expected downstream reward of compressor $\mathcal{C}$,
$\mathcal{J}(\mathcal{C})
=\mathbb{E}_{u\sim\mathcal{D},
\,\tau\sim p_{\mathcal{C},B}(\cdot\mid u)}[R(\tau)]$ \\

$\mathcal{C}^{*}$ &
Optimal compressor in Problem~\ref{def:problem} \\

\midrule
\multicolumn{2}{c}{\textit{Boundary-Level Counterfactual Evaluation}}\\

$\xi_t$ &
Boundary state associated with a realized compression,
$\xi_t=(h_{t+1},\bar z_{t+1})$ \\

$x$ &
Candidate context from which execution is continued, typically
$x\in\{\bar z_{t+1},c_{t+1}\}$ \\

$\varnothing$ &
Reference execution without further compression after the evaluated
boundary \\

$Q_{\varnothing}^{R}(\xi_t,x)$ &
Expected terminal reward when execution continues from context $x$ at
boundary state $\xi_t$ without further compression \\

$Q_{\varnothing}^{L}(\xi_t,x)$ &
Expected number of remaining interaction steps under the same
no-further-compression continuation \\

$\mathbf{D}_{\varnothing}(\xi_t,c_{t+1})$ &
Boundary-level compression effect, consisting of outcome hazard and
interaction burden \\

$H_{\varnothing}(\xi_t,c_{t+1})$ &
Outcome hazard:
$Q_{\varnothing}^{R}(\xi_t,\bar z_{t+1})
-Q_{\varnothing}^{R}(\xi_t,c_{t+1})$ \\

$B_{\varnothing}(\xi_t,c_{t+1})$ &
Interaction burden:
$Q_{\varnothing}^{L}(\xi_t,c_{t+1})
-Q_{\varnothing}^{L}(\xi_t,\bar z_{t+1})$ \\

$m$ &
Number of independent counterfactual continuations per condition in a
fixed-budget boundary evaluation \\

$\tau_{t,i}^{x}$ &
The $i$-th continuation rollout from context $x$ at boundary $t$ \\

$\ell(\tau)$ &
Number of interaction steps in continuation rollout $\tau$ \\

$\widehat{\mathbf{D}}_{\varnothing}
=(\widehat H_{\varnothing},\widehat B_{\varnothing})$ &
Empirical boundary effect estimated from counterfactual continuations \\

$\mathcal{B}(\tau)$ &
Set of compression boundaries encountered in rollout $\tau$ \\

\midrule
\multicolumn{2}{c}{\textit{Prompt Adaptation with \texttt{PAIR}}}\\

$P$ &
Natural-language compression prompt \\

$\pi_P(c\mid x,B)$ &
Compression policy induced by prompting the frozen compressor model
with $P$ \\

$P_0$ &
Original structured compression prompt used to collect boundary
evidence \\

$\mathcal{C}_{P_0}$ &
Compressor induced by the original prompt $P_0$ \\

$\mathcal{D}_{\mathrm{train}}$ &
Training-task set used for boundary mining and prompt selection \\

$r$ &
Successive-halving round index, $r\in\{1,2,3\}$ \\

$\widehat H_t^{(r)},\,\widehat B_t^{(r)}$ &
Running outcome-hazard and interaction-burden estimates at boundary
$t$, computed from the first $r$ PRE/POST continuation pairs \\

$\tau_H,\,\tau_B$ &
Predefined thresholds for outcome hazard and interaction burden,
respectively \\

$s_t^{(r)}$ &
Normalized harm score used for successive halving,
$s_t^{(r)}
=\max\{\widehat H_t^{(r)}/\tau_H,
\widehat B_t^{(r)}/\tau_B\}$ \\

$\widehat H_t,\,\widehat B_t$ &
Final empirical outcome hazard and interaction burden computed from all
continuations allocated to surviving boundary $t$ \\

$\mathcal{P}_{\mathrm{cand}}
=\{P^{(1)},\ldots,P^{(5)}\}$ &
Candidate compression prompts produced from the aggregated boundary
diagnoses \\

$\mathrm{Pass}^{k}(P)$ &
Fraction of tasks completed successfully in all $k$ independent runs
under prompt $P$ \\

$\mathrm{Pass}(P)$ &
Pass rate used for end-to-end candidate-prompt selection \\

$\mathrm{Steps}(P)$ &
Mean number of interaction steps under prompt $P$ \\

$P^{*}$ &
Selected prompt, chosen lexicographically by highest
$\mathrm{Pass}(P)$ and then lowest $\mathrm{Steps}(P)$ \\

\end{longtable}
\end{center}

\section{Theoretical Proofs}
\subsection{Proof of Theorem~\ref{thm:boundary-performance-difference}}
\label{app:boundary-effect-proof}

\textbf{Theorem~\ref{thm:boundary-performance-difference}} (Boundary performance-difference identity).
\textit{Let $\varnothing$ denote execution without further compression. For any compressor $\mathcal C$,
\begin{equation}
\mathcal J(\mathcal C)-\mathcal J(\varnothing)
=
-\mathbb E_{\substack{
u\sim\mathcal D,\;
\tau\sim p_{\mathcal C,B}(\cdot\mid u)
}}
\!\left[
\sum_{t\in\mathcal B(\tau)}
H_{\varnothing}(\xi_t,c_{t+1})
\right].
\label{app-eq:boundary-performance-difference}
\end{equation}
Thus, minimizing expected cumulative boundary hazard is equivalent to maximizing the objective in Problem~1.}

\begin{proof}
Fix a task $u$ and let
\[
V_{\varnothing}(\xi_t)
=
Q_{\varnothing}^R(\xi_t,\bar z_{t+1})
\]
denote the expected return when execution continues from boundary state
$\xi_t$ without further compression. By definition,
\begin{equation}
-H_{\varnothing}(\xi_t,c_{t+1})
=
Q_{\varnothing}^R(\xi_t,c_{t+1})
-
V_{\varnothing}(\xi_t).
\label{eq:no-compression-advantage}
\end{equation}

Consider a rollout $\tau\sim p_{\mathcal C,B}(\cdot\mid u)$ with ordered
compression boundaries $t_1<\cdots<t_K$, and write
$s_j=\xi_{t_j}$ and $c_j=c_{t_j+1}$.
Let $s_{K+1}$ denote the terminal state and set
$V_{\varnothing}(s_{K+1})=R(\tau)$.
After applying $c_j$, execution is identical to the no-compression
continuation until the next compression boundary. Hence,
\begin{equation}
Q_{\varnothing}^R(s_j,c_j)
=
\mathbb E\!\left[
V_{\varnothing}(s_{j+1})
\mid s_j,c_j
\right].
\end{equation}
Using Equation~\eqref{eq:no-compression-advantage}, the tower property,
and linearity of expectation,
\begin{align}
-\mathbb E_{\tau\sim p_{\mathcal C,B}(\cdot\mid u)}
\left[
\sum_{j=1}^{K} H_{\varnothing}(s_j,c_j)
\right]
&=
\mathbb E_{\tau\sim p_{\mathcal C,B}(\cdot\mid u)}
\left[
\sum_{j=1}^{K}
\bigl(
V_{\varnothing}(s_{j+1})
-
V_{\varnothing}(s_j)
\bigr)
\right]
\nonumber\\
&=
\mathbb E_{\tau\sim p_{\mathcal C,B}(\cdot\mid u)}
\left[
\mathbf 1_{\{K>0\}}
\bigl(
R(\tau)-V_{\varnothing}(s_1)
\bigr)
\right].
\label{eq:no-compression-telescoping}
\end{align}

Before the first compression boundary, execution under $\mathcal C$ and
under no compression is identical. Therefore, the event $\{K>0\}$, the
distribution of $s_1$, and the return on trajectories with $K=0$ are
the same in both systems. Consequently,
\begin{equation}
\mathbb E_{\tau\sim p_{\mathcal C,B}(\cdot\mid u)}
\left[
\mathbf 1_{\{K>0\}}V_{\varnothing}(s_1)
\right]
=
\mathbb E_{\tau\sim p_{\varnothing}(\cdot\mid u)}
\left[
\mathbf 1_{\{K>0\}}R(\tau)
\right].
\end{equation}
The $K=0$ terms cancel, so Equation~\eqref{eq:no-compression-telescoping}
gives
\[
-\mathbb E_{\tau\sim p_{\mathcal C,B}(\cdot\mid u)}
\left[
\sum_{t\in\mathcal B(\tau)}
H_{\varnothing}(\xi_t,c_{t+1})
\right]
=
\mathcal J_u(\mathcal C)-\mathcal J_u(\varnothing).
\]
Taking expectation over $u\sim\mathcal D$ proves
Equation~\eqref{app-eq:boundary-performance-difference}.
\end{proof}

\section{Additional Experiment Results}

\subsection{Extra Experimental Settings}\label{app:experimental_details}

\noindent\textbf{Models.}
We run every agent, compressor, and \texttt{PAIR} optimizer call through an
OpenAI-compatible chat interface, with one backbone per experiment family.
In the main experiments the backbone is \textsc{GPT-5.6-Luna}, accessed via
the OpenAI Responses API with reasoning effort set to \texttt{medium}; the
agent acts through a single native function call
(\texttt{execute\_python} on AppWorld, \texttt{execute\_action} on
OfficeBench, the domain tools on $\tau^2$-bench) with at most 2{,}048 output
tokens per step, and the same model serves as the history compressor (at most
8{,}192 output tokens per compression) and as the \texttt{PAIR} optimizer.
In the transfer experiments the backbone is \textsc{MiniMax-M3}, served
through Ollama's hosted OpenAI-compatible endpoint. Compression prompts are
identical across backbones; only the model behind them changes. On
$\tau^2$-bench the simulated customer is the benchmark's default user
simulator, \texttt{gpt-4.1-2025-04-14} at temperature $0$, so run-to-run
variation comes from the agent side only. Each reported number averages three
independent runs with different seeds.

\textbf{Splits.}
For AppWorld and $\tau^2$-Bench Retail, we use the official training and test splits.
For OfficeBench and 8-objective QA, we use the stratified training and test splits provided in the ACON repository.
\textbf{Context Budgets.}
To induce approximately two to three compression events per task on average, we set the context budget to 4,096 tokens for AppWorld and 8-objective QA, and to 2,048 tokens for OfficeBench and $\tau^2$-Bench Retail.

\subsection{Boundary Effects on Additional Benchmarks}
\label{app:additional-boundary-effects}

We repeat the boundary-level analysis on OfficeBench and
$\tau^2$-Bench Retail using $m=9$ continuations for each pre- and
post-compression condition. The analysis covers 211 compression
boundaries from 71 OfficeBench trajectories and 200 boundaries from
67 $\tau^2$-Bench trajectories.

\begin{figure}[!h]
    \centering
    \captionsetup[subfigure]{font=small,skip=2pt}

    \begin{subfigure}[b]{0.35\textwidth}
        \centering
        \includegraphics[width=\linewidth]
        {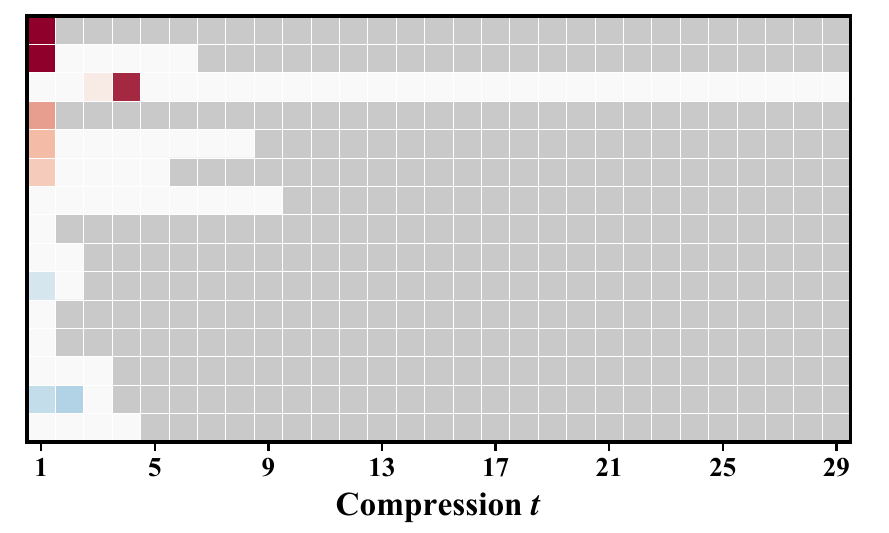}
        \caption{Failed trajectories.}
        \label{fig:officebench-hazard-failed}
    \end{subfigure}
    \hfill
    \begin{subfigure}[b]{0.20\textwidth}
        \centering
        \includegraphics[width=\linewidth]
        {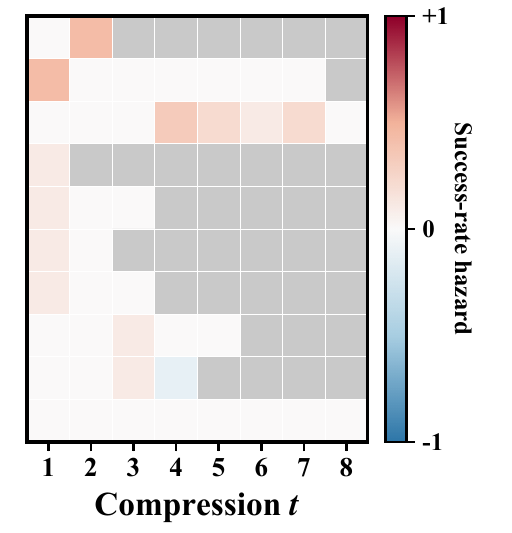}
        \caption{Successful trajectories.}
        \label{fig:officebench-hazard-successful}
    \end{subfigure}
    \hfill
    \begin{subfigure}[b]{0.39\textwidth}
        \centering
        \includegraphics[width=\linewidth]
        {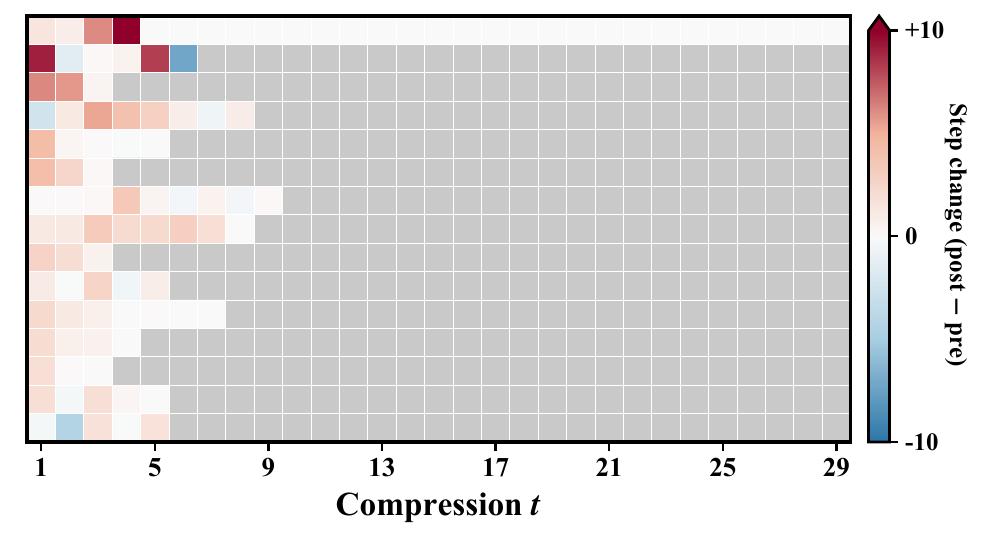}
        \caption{Interaction burden.}
        \label{fig:officebench-step-burden}
    \end{subfigure}

    \caption{
        \textbf{Boundary-level compression effects on OfficeBench.}
        Panel (a) includes all failed trajectories. Panels (b) and
        (c) show the ten and fifteen trajectories with the largest
        outcome hazard and interaction burden, respectively.
        Red denotes degradation or additional steps, blue denotes
        improvement, and gray denotes absent boundaries.
    }
    \label{fig:officebench-boundary-effects}
\end{figure}

\begin{figure}[!h]
    \centering
    \captionsetup[subfigure]{font=small,skip=2pt}

    \begin{subfigure}[b]{0.29\textwidth}
        \centering
        \includegraphics[width=\linewidth]
        {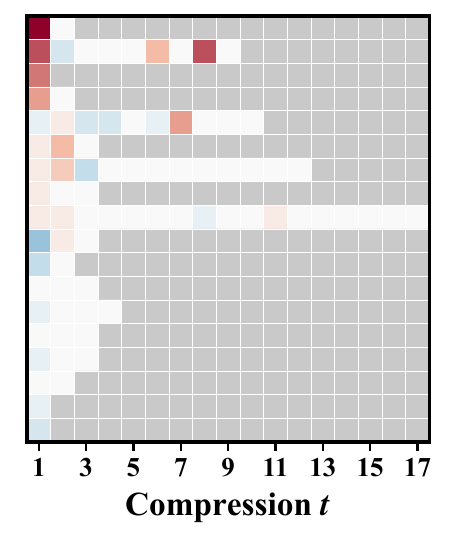}
        \caption{Failed trajectories.}
        \label{fig:tau2-hazard-failed}
    \end{subfigure}
    \hfill
    \begin{subfigure}[b]{0.24\textwidth}
        \centering
        \includegraphics[width=\linewidth]
        {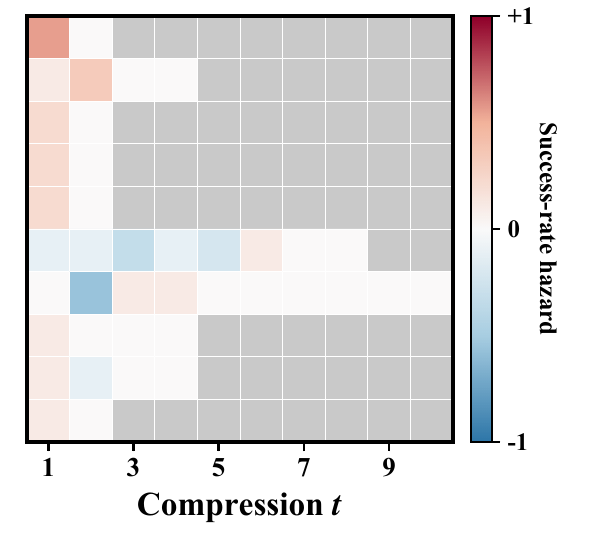}
        \caption{Successful trajectories.}
        \label{fig:tau2-hazard-successful}
    \end{subfigure}
    \hfill
    \begin{subfigure}[b]{0.41\textwidth}
        \centering
        \includegraphics[width=\linewidth]
        {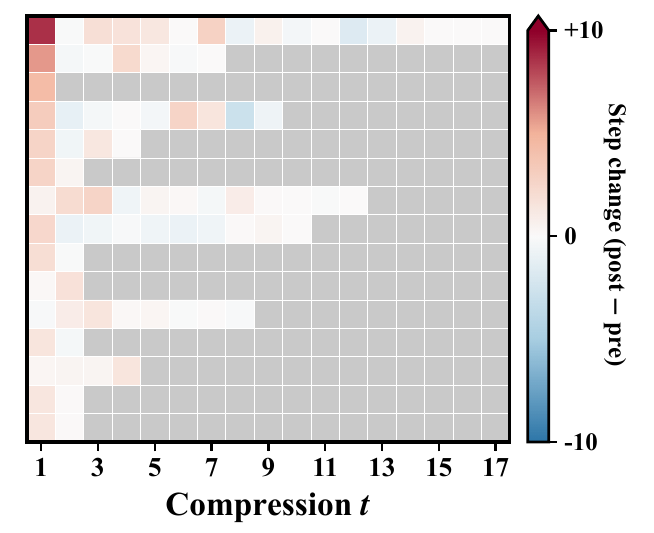}
        \caption{Interaction burden.}
        \label{fig:tau2-step-burden}
    \end{subfigure}

    \caption{
        \textbf{Boundary-level compression effects on
        $\tau^2$-Bench Retail.}
        Panel (a) includes all failed trajectories. Panels (b) and
        (c) show the ten and fifteen trajectories with the largest
        outcome hazard and interaction burden, respectively.
        Colors follow Figure~\ref{fig:officebench-boundary-effects}.
    }
    \label{fig:tau2-boundary-effects}
\end{figure}

Figures~\ref{fig:officebench-boundary-effects}
and~\ref{fig:tau2-boundary-effects} show the same qualitative pattern
as AppWorld: most compressions leave continuation success unchanged,
while severe effects are concentrated at a small number of boundaries.
On OfficeBench, 184 of 211 boundaries leave continuation success
unchanged. Four satisfy $\widehat{H}_t \geq 0.5$ and seven satisfy
$\widehat{B}_t \geq 5$, yielding nine retained boundaries across six
trajectories after accounting for overlap. On $\tau^2$-Bench Retail,
141 of 200 boundaries leave continuation success unchanged. Seven
satisfy the outcome-hazard threshold and two satisfy the burden
threshold, yielding nine retained boundaries across eight trajectories.

Moderate interaction burden is more common. Compression increases
continuation length at 109 OfficeBench boundaries and 95
$\tau^2$-Bench Retail boundaries. Nevertheless, only seven and two
boundaries, respectively, add at least five steps. Severe outcome
degradation and interaction burden are therefore localized, although
smaller execution costs occur more broadly across compressed
trajectories.

Outcome hazard and interaction burden are only weakly coupled.
Their boundary-level Spearman correlations are $0.004$, $0.102$,
and $-0.054$ on AppWorld, OfficeBench, and $\tau^2$-Bench Retail,
respectively. Compression can therefore increase recovery effort
without immediately reducing continuation success, motivating the
use of both signals.

\subsection{Breakdown on OfficeBench}
\label{app:officebench-breakdown}

Table~\ref{tab:officebench_gpt_luna} reports results by the number of
applications involved in each task.

\begin{table*}[h!]
\centering
\caption{Results on \textbf{OfficeBench} , grouped by the number of applications
involved in each task. Acc. is mean task
success, with across-run standard deviation shown as a subscript;
Pass$^3$ is the fraction of tasks solved in all three runs; Steps and
Peak are mean interaction steps and peak input length ($10^3$ tokens).
\textit{No compression} is the uncompressed reference. History-only
and prefix-conditioned compression differ in whether the compressor
receives the fixed prefix, which remains separately visible to the
agent in both settings. Bold denotes the best compressed result within
each block; gray rows indicate \texttt{PAIR}.}
\vspace{-0.1in}
\label{tab:officebench_gpt_luna}
\resizebox{\textwidth}{!}{
\begin{tabular}{lccccccccccccc}
\toprule
\multirow{2}{*}{Method} & 
\multicolumn{3}{c}{Average (95)} & 
\multicolumn{3}{c}{1-APP (42)} & 
\multicolumn{3}{c}{2-APP (22)} & 
\multicolumn{3}{c}{3-APP (31)} \\
\cmidrule(lr){2-4} \cmidrule(lr){5-7} \cmidrule(lr){8-10} \cmidrule(lr){11-13}
& Acc. $\uparrow$ & Pass$^3$ $\uparrow$ & Steps $\downarrow$
& Acc. $\uparrow$ & Pass$^3$ $\uparrow$ & Steps $\downarrow$
& Acc. $\uparrow$ & Pass$^3$ $\uparrow$ & Steps $\downarrow$
& Acc. $\uparrow$ & Pass$^3$ $\uparrow$ & Steps $\downarrow$\\
\midrule
\rowcolor{gray!30}\multicolumn{13}{c}{\textbf{Agent:} \texttt{GPT-5.6 Luna} / \textbf{Compressor:} \texttt{GPT-5.6 Luna}}\\
\midrule
\textit{No compression} & \textit{82.1}\std{4.8} & \textit{75.8} & \textit{11.2} & \textit{92.9}\std{4.1} & \textit{88.1} & \textit{7.6} & \textit{81.8}\std{4.5} & \textit{77.3} & \textit{10.9} & \textit{67.7}\std{8.5} & \textit{58.1} & \textit{16.2} \\
\midrule[0.1pt]
\multicolumn{13}{l}{\textbf{History-Only Compression}}\\
FIFO & 62.1\std{1.8} & 54.7 & 23.2 & 81.7\std{1.4} & 73.8 & 10.3 & 56.1\std{2.6} & 50.0 & 28.2 & 39.8\std{3.7} & 32.3 & 37.0 \\
LLMLingua-2 & 76.1\std{1.2} & 68.4 & 14.4 & 89.7\std{1.4} & 83.3 & \textbf{8.0} & 81.8\std{0.0} & \textbf{77.3} & 14.0 & 53.8\std{4.9} & 41.9 & 23.2 \\
Prompting & 76.5\std{2.4} & 65.3 & \textbf{11.9} & 91.3\std{1.4} & 83.3 & 8.1 & 74.2\std{6.9} & 63.6 & \textbf{10.9} & 58.1\std{6.5} & 41.9 & 17.8 \\
ACON-UT & 77.2\std{1.6} & 68.4 & 13.3 & 91.3\std{1.4} & 85.7 & 9.0 & 78.8\std{2.6} & 72.7 & 12.7 & 58.1\std{3.2} & 45.2 & 19.6 \\
ACON-UTCO & 73.0\std{5.8} & 60.0 & 13.6 & 84.1\std{9.0} & 71.4 & 8.2 & 78.8\std{5.2} & 59.1 & 12.7 & 53.8\std{3.7} & 45.2 & 21.6 \\
\rowcolor{gray!15}\texttt{PAIR} (Ours) & \textbf{81.1}\std{2.8} & \textbf{70.5} & \textbf{11.9} & \textbf{92.9}\std{2.4} & \textbf{85.7} & 8.1 & \textbf{84.8}\std{2.6} & 72.7 & 11.8 & \textbf{62.4}\std{4.9} & \textbf{48.4} & \textbf{17.2} \\
\midrule
\multicolumn{13}{l}{\textbf{Prefix-Conditioned Compression}}\\
Prompting & 78.9\std{2.8} & 66.3 & 13.3 & 89.7\std{3.6} & 81.0 & 9.1 & 81.8\std{2.0} & 63.6 & 13.0 & 62.4\std{4.9} & 48.4 & 19.3 \\
ACON-UT & 75.1\std{0.6} & 61.1 & 15.5 & 85.7\std{4.1} & 71.4 & 9.9 & 83.3\std{5.2} & 72.7 & 15.2 & 54.8\std{3.2} & 38.7 & 23.4 \\
ACON-UTCO & 78.9\std{3.8} & 65.3 & 13.6 & 90.5\std{4.1} & 81.0 & 9.1 & \textbf{84.8}\std{2.6} & \textbf{68.2} & 13.1 & 59.1\std{4.9} & 41.9 & 20.2 \\
\rowcolor{gray!15}\texttt{PAIR} (Ours) & \textbf{81.4}\std{0.6} & \textbf{72.6} & \textbf{11.8} & \textbf{92.9}\std{2.4} & \textbf{88.1} & \textbf{8.5} & 81.8\std{0.0} & \textbf{68.2} & \textbf{11.2} & \textbf{65.6}\std{3.7} & \textbf{54.8} & \textbf{16.7} \\
\bottomrule
\end{tabular}
}
\end{table*}

The effect of compression becomes more pronounced as workflows span
more applications. Under no compression, accuracy decreases from
$92.9\%$ on 1-APP tasks to $67.7\%$ on 3-APP tasks, while the mean
execution length more than doubles. \texttt{PAIR}'s gains are likewise
largest on multi-application workflows. On 3-APP tasks,
prefix-conditioned \texttt{PAIR} improves over Prompting from
$62.4\%$ to $65.6\%$ accuracy and from $48.4\%$ to $54.8\%$ Pass$^3$,
while reducing Steps from $19.3$ to $16.7$. It also approaches the
uncompressed reference of $67.7\%$ accuracy, $58.1\%$ Pass$^3$, and
$16.2$ Steps. These results indicate that boundary-guided adaptation is
most beneficial when compression must preserve execution state across
longer, cross-application workflows.

\subsection{Extra Experiment on 8-Objective QA}

We also evaluate on \textbf{8-objective QA}~\citep{kwiatkowski2019natural, zhou2026mem}, a QA benchmark where agents use a search tool to answer eight questions and return a consolidated answer set.

\begin{table}[h!]
\centering
\scriptsize
\caption{
Results on \textbf{8-objective QA}, aggregated over three runs.
EM is mean exact match, and EM$^3$ is the fraction of objectives
answered exactly in all three runs. Peak and Total denote peak input
length and total agent--compressor token use per episode, respectively,
in $10^3$ tokens. \textit{No compression} is the uncompressed
reference. Bold denotes the best result within each block; gray rows
indicate \texttt{PAIR}.
}
\vspace{-0.1in}
\label{tab:qa_main}

\resizebox{0.56\columnwidth}{!}{%
\setlength{\tabcolsep}{3.5pt}
\begin{tabular}{lcccccc}
\toprule
Method
& EM $\uparrow$
& EM$^3$ $\uparrow$
& F1 $\uparrow$
& Steps $\downarrow$
& Peak $\downarrow$
& Total $\downarrow$ \\
\midrule
\textit{No compression}
& \textit{40.9}\std{1.8}
& \textit{29.8}
& \textit{54.2}
& \textit{18.5}
& \textit{14.53}
& \textit{181.83} \\
\midrule
\multicolumn{7}{l}{\textbf{History-Only Compression}} \\
FIFO
& 14.0\std{1.4} & 3.6 & 18.8 & 45.1 & \textbf{6.70} & 274.03 \\
LLMLingua-2
& 35.8\std{0.6} & 25.9 & 48.7 & 28.0 & 7.38 & 353.61 \\
Prompting
& 37.2\std{1.6} & 24.6 & 51.8 & 26.6 & 7.34 & 177.48 \\
ACON-UT
& 38.5\std{1.4} & 28.8 & 53.4 & \textbf{19.6} & 7.43 & \textbf{142.00} \\
ACON-UTCO
& 35.3\std{1.3} & 21.8 & 48.5 & 34.1 & 7.37 & 202.99 \\
\rowcolor{gray!15}
\texttt{PAIR} (Ours)
& \textbf{39.3}\std{0.5}
& \textbf{29.4}
& \textbf{53.7}
& 27.0
& 7.32
& 175.51 \\
\midrule
\multicolumn{7}{l}{\textbf{Prefix-Conditioned Compression}} \\
Prompting
& 42.7\std{1.2} & 33.8 & 55.9 & 22.5 & 7.39 & 159.53 \\
ACON-UT
& 43.0\std{1.1}
& 34.5
& 56.8
& \textbf{17.6}
& \textbf{7.27}
& \textbf{118.97} \\
ACON-UTCO
& 40.8\std{0.4} & 32.4 & 55.1 & 22.5 & 7.32 & 138.36 \\
\rowcolor{gray!15}
\texttt{PAIR} (Ours)
& \textbf{43.7}\std{0.3}
& \textbf{35.8}
& \textbf{57.5}
& 20.9
& 7.33
& 145.49 \\
\bottomrule
\end{tabular}%
}
\vspace{-3mm}
\end{table}

\noindent\textbf{The adapted prompts emphasize different functions of
a compressed checkpoint.}
A checkpoint simultaneously serves as a representation of past
execution and as input that can steer future agent behavior. Both
ACON-UT and \texttt{PAIR} affect these functions, but their adapted
rules place different emphasis on them.

\begin{promptbox}{ACON-UT: structuring downstream execution}
Maintain a State Table containing authentication state, identifiers,
retrieved values, pagination state, and instructions for reconstructing
session variables.

Maintain a Completed/Pending Action Ledger with the status, evidence,
and next action for every operation.

Before an irreversible action, require a preflight check covering the
target set, exclusions, credentials, content, and endpoint parameters.

Call the completion endpoint only after all completion checks pass.
\end{promptbox}

The ACON-UT guideline preserves useful state, but it also specifies how
that state should organize subsequent execution. Its state table tells
the next agent which variables to reconstruct; its action ledger assigns
status and a next operation to each item; and its preflight and
completion rules determine how side effects and task termination should
be handled. These are not merely decisions about which historical facts
to retain. They turn the generated checkpoint into an operational
interface for controlling the downstream agent.

\begin{promptbox}{\texttt{PAIR}: preserving the execution state}
Preserve the complete user goal and explicit constraints, including
exact predicates, qualifiers, versions, dates, jurisdictions, and
target count.

Reconcile the previous summary against the newest observed results
rather than preserving contradicted or stale state.

Move an item to Done only when the relevant answer or state is
supported; otherwise record it as attempted but unresolved.

Preserve exact file paths, function names, error messages, query
parameters, and other continuation-relevant literals.
\end{promptbox}

The \texttt{PAIR} guideline is centered more directly on the
construction of the compressed state. Its rules determine which task
semantics must remain invariant, how new evidence updates previously
recorded state, and how uncertainty and incomplete work should be
represented. It necessarily retains pending actions because they are
part of the execution state, but it derives them from the observed
trajectory rather than introducing a separate task-solving procedure.
Thus, the distinction is not that only one prompt affects downstream
behavior; any checkpoint can do so. Rather, ACON-UT explicitly uses the
checkpoint as a carrier of execution policy, whereas \texttt{PAIR}
more narrowly adapts how the compressor represents the state on which
the original agent policy operates.

This difference is reflected in the pattern of improvements in
Table~\ref{tab:qa_main}. Under history-only compression,
\texttt{PAIR} improves its starting prompt by $2.1\%$ EM, $4.8\%$
EM$^3$, and $1.9\%$ F1, while ACON-UT improves the same metrics by
$1.3\%$, $4.2\%$, and $1.6\%$. Under prefix-conditioned compression,
the corresponding gains are $1.0\%$, $2.0\%$, and $1.6\%$ for
\texttt{PAIR}, compared with $0.3\%$, $0.7\%$, and $0.9\%$ for
ACON-UT. Across both scopes, \texttt{PAIR} therefore produces the
larger improvement in downstream answer quality and attains the highest
EM, EM$^3$, and F1.

\noindent\textbf{ACON-UT's behavioral rules provide a direct mechanism
for reducing interaction steps.}
The different emphasis of the two guidelines also helps explain their
execution lengths. ACON-UT adds explicit rules governing how the
downstream agent should handle low-yield objectives and decide when to
stop.

\begin{promptbox}{ACON-UT: retry and completion policy}
After two semantically equivalent or non-informative searches for the
same issue, mark the issue as low-yield and retain the strongest
supported candidate or the explicit unresolved alternatives.

Stop research when every answer slot has a supported candidate or a
clearly labeled best-supported unresolved value.
\end{promptbox}

These rules provide two direct routes to shorter execution. First, the
retry limit suppresses additional searches once two attempts are judged
semantically redundant or uninformative. Second, the completion
condition allows the task to terminate without resolving every
objective: a best-supported unresolved value is sufficient to close an
answer slot. ACON-UT therefore changes not only the retained state but
also the downstream agent's search horizon and stopping criterion.

\begin{promptbox}{\texttt{PAIR}: preserving unresolved state}
Do not promote an answer, attribution, equivalence, or result that the
raw messages do not establish.

A search may be marked done even when it failed to resolve an answer,
but its unresolved outcome must remain explicit.

Retain unresolved work with the best supported state and the missing
evidence needed for resolution.
\end{promptbox}

The \texttt{PAIR} rules make a different distinction: completing a
search action does not imply that the underlying objective has been
resolved. An unsuccessful attempt is preserved as such, and its current
candidate cannot be upgraded beyond the available evidence. Because
\texttt{PAIR} introduces neither a fixed retry limit nor an alternative
completion criterion, the original downstream agent may continue with a
materially different search when execution budget remains.

The resulting step pattern is consistent across both compression
scopes. Under history-only compression, ACON-UT reduces execution from
$26.6$ steps with the initial prompt to $19.6$, whereas \texttt{PAIR}
uses $27.0$ steps. Under prefix-conditioned compression, ACON-UT reduces
execution from $22.5$ to $17.6$ steps, compared with $20.9$ for
\texttt{PAIR}. The repeated reduction is consistent with ACON-UT's
explicit retry and stopping rules. This is a valid route to lower
execution cost, but it is conceptually distinct from improving the
quality of the compressed state.

Taken together, the results reveal a quality--control trade-off.
ACON-UT obtains larger step reductions by using the checkpoint to
modify how the downstream agent searches and terminates.
\texttt{PAIR} concentrates its adaptation on preserving task and
evidence state, yielding the strongest EM, EM$^3$, and F1 under both
compression scopes. Consequently, ACON-UT's lower step count should
not be interpreted as evidence of better compression alone: part of
the gain follows from an additional behavioral policy embedded in the
generated checkpoint.

\subsection{Template-Controlled Comparison with ACON-UT}
\label{app:template-controlled-acon}

The original ACON proposer is instructed to rewrite the compression
prompt and add concrete rules for retaining essential state, but it
does not constrain the output schema. Candidate prompts may therefore
preserve, extend, or replace the section structure of the initial
OpenClaw prompt. Because \texttt{PAIR} keeps this structure fixed, the
additional flexibility available to ACON-UT could otherwise confound
the comparison.

ACON generates five candidate prompts, nominally from different random
subsets of regression examples. In our settings, however, the number
of available examples is smaller than its configured sample size of
30. The sampler consequently uses $k=\min(30,n)=n$, causing all five
candidates within a setting to receive the same complete evidence set.
We verified that their proposer inputs have identical hashes. Thus,
whether a candidate changes the checkpoint structure is not determined
by different evidence; it arises from stochastic proposer generation,
after which end-to-end validation selects one of the five candidates.

This observation also explains why we perform the structural ablation
only on AppWorld. The ACON-UT prompts selected on OfficeBench already
retain the original OpenClaw section skeleton used by \texttt{PAIR}, so
the reported OfficeBench comparison is already structure-controlled.
On AppWorld, in contrast, the selected ACON-UT prompts reorganize the
checkpoint schema. We therefore rerun ACON-UT on AppWorld while
requiring every candidate to retain the OpenClaw section structure.
The locked variant uses the same initial prompt, trajectory-level
analysis artifact, optimizer, candidate count, adaptation tasks, and
end-to-end selection protocol as the original ACON-UT run. Only the
instructions within the existing sections may be revised.

\begin{table}[h]
\centering
\scriptsize
\caption{
\textbf{Template-controlled comparison on AppWorld.}
ACON-UT (Locked) retains the OpenClaw section structure used by
\texttt{PAIR}. A separate OfficeBench ablation is unnecessary because
its selected ACON-UT prompts already preserve this structure. Results
are aggregated over three runs. Peak is measured in $10^3$ input
tokens. Bold denotes the best result within each scope; gray rows
indicate \texttt{PAIR}.
}
\vspace{-0.1in}
\label{tab:template-controlled-acon}

\resizebox{0.48\columnwidth}{!}{%
\setlength{\tabcolsep}{3.5pt}
\renewcommand{\arraystretch}{0.95}
\begin{tabular}{@{}lcccc@{}}
\toprule
Method
& Acc. $\uparrow$
& Pass$^3$ $\uparrow$
& Steps $\downarrow$
& Peak $\downarrow$ \\
\midrule
\multicolumn{5}{l}{\textbf{History-Only Compression}} \\
Prompting
& 78.6\std{2.7} & 68.5 & 17.0 & 9.23 \\
ACON-UT
& 81.2\std{1.2} & 70.8 & 18.9 & 9.36 \\
ACON-UT (Locked)
& 79.4\std{2.7} & 70.8 & \textbf{16.6} & 9.08 \\
\rowcolor{gray!15}
\texttt{PAIR} (Ours)
& \textbf{84.7}\std{1.7}
& \textbf{79.8}
& \textbf{16.6}
& \textbf{9.03} \\
\midrule
\multicolumn{5}{l}{\textbf{Prefix-Conditioned Compression}} \\
Prompting
& 81.0\std{1.6} & 72.0 & 15.3 & 9.51 \\
ACON-UT
& 81.3\std{0.9} & 72.6 & 17.3 & 10.21 \\
ACON-UT (Locked)
& 83.5\std{1.8} & 78.0 & 15.3 & 9.66 \\
\rowcolor{gray!15}
\texttt{PAIR} (Ours)
& \textbf{85.9}\std{2.4}
& \textbf{80.4}
& \textbf{15.2}
& \textbf{9.47} \\
\bottomrule
\end{tabular}%
}
\vspace{-3mm}
\end{table}

\noindent\textbf{Template flexibility does not account for ACON-UT's
gains.}
As shown in Table~\ref{tab:template-controlled-acon}, locking the
template leaves history-only Pass$^3$ unchanged at $70.8\%$. Accuracy
changes from $81.2\%$ to $79.4\%$, while interaction length decreases
from $18.9$ to $16.6$ steps and peak context decreases from $9.36$K to
$9.08$K tokens. Under prefix-conditioned compression, the locked
variant improves accuracy from $81.3\%$ to $83.5\%$ and Pass$^3$ from
$72.6\%$ to $78.0\%$, while reducing interaction length from $17.3$ to
$15.3$ steps and peak context from $10.21$K to $9.66$K tokens. The
ability to redesign the checkpoint is therefore not necessary for
ACON-UT to improve over its starting prompt. In this setting, fixing
the interface preserves or improves reliability while reducing
execution overhead.

\noindent\textbf{\texttt{PAIR}'s advantage persists under the matched
interface.}
With the section structure controlled, history-only \texttt{PAIR}
achieves $84.7\%$ accuracy and $79.8\%$ Pass$^3$, compared with
$79.4\%$ and $70.8\%$ for locked ACON-UT. The two methods require the
same $16.6$ interaction steps, and their peak contexts are nearly
identical. Under prefix-conditioned compression, \texttt{PAIR} reaches
$85.9\%$ accuracy and $80.4\%$ Pass$^3$, compared with $83.5\%$ and
$78.0\%$ for locked ACON-UT, while requiring slightly fewer steps and a
shorter peak context. The remaining performance differences therefore
cannot be explained by a more favorable checkpoint schema or a larger
prompt-design space.

The controlled comparison isolates the substantive difference between
the methods: the evidence used for adaptation. ACON-UT derives its
revisions from differences between complete successful and failed
trajectories. This signal contains differences in retained information,
planning, and stochastic execution, allowing trajectory-specific
procedures to become general compression rules even when the output
schema is fixed. \texttt{PAIR} instead uses paired continuations to
identify information whose removal changes success or induces recovery
at a particular compression boundary. The results are consistent with
this localized evidence producing more targeted state-preservation
rules, especially under history-only compression, where the reliability
gap remains substantial after the interface is matched.

This ablation does not imply that checkpoint structure is universally
irrelevant. Rather, it shows that structural flexibility is neither
required for ACON-UT's effectiveness nor sufficient to explain
\texttt{PAIR}'s gains. Template locking controls the interface, while
the distinction between trajectory-level planning induction and
boundary-level compression diagnosis remains.

\subsection{Sensitivity of Boundary Evaluation}
\label{app:boundary-sensitivity}

We examine the sensitivity of boundary evaluation to continuation
sampling and boundary-selection thresholds. All analyses use training
trajectories only and involve no additional prompt generation or
test-set evaluation.

\noindent\textbf{Continuation count.}
We first study the stability of fixed-budget boundary estimates. For
each boundary, we subsample
$m\in\{1,3,5,9\}$ continuations per condition from the nine available
PRE and POST rollouts. We repeat subsampling 1,000 times and compare
each estimate with the full nine-draw estimate, which serves as a
higher-sample reference rather than ground truth.

\begin{figure}[h]
    \centering
    \includegraphics[width=0.92\linewidth]
    {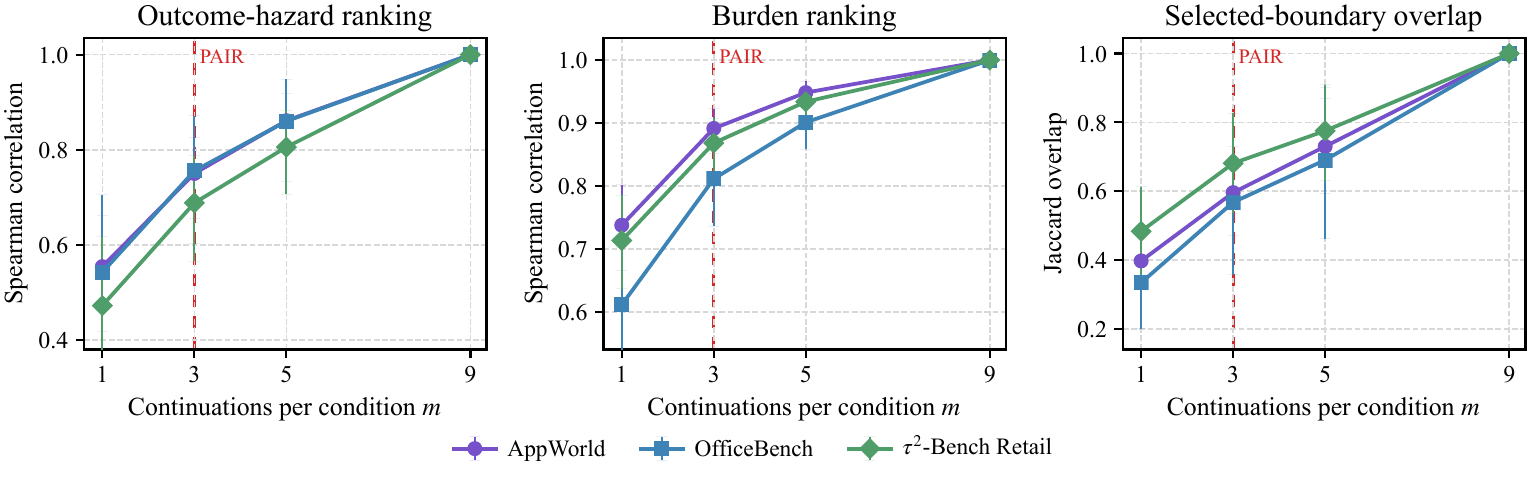}
    \caption{
        \textbf{Sensitivity to the number of continuations.}
        The first two panels report the rank correlation of outcome
        hazard and interaction burden with their nine-draw estimates;
        the third reports the Jaccard overlap of the resulting boundary
        selections. Points and error bars show the mean and
        2.5--97.5 percentiles over 1,000 subsamples. The red dashed line
        marks the three-pair cap used by \texttt{PAIR}'s
        successive-halving evaluation.
    }
    \label{fig:continuation-count-sensitivity}
    \vspace{-3mm}
\end{figure}

As shown in Figure~\ref{fig:continuation-count-sensitivity}, stability
improves consistently with additional continuations. With three
continuations per condition, rank correlations with the nine-draw
estimates range from $0.69$ to $0.76$ for outcome hazard and from
$0.81$ to $0.89$ for interaction burden, while selected-boundary
overlap ranges from $0.57$ to $0.68$. Increasing $m$ to five improves
these ranges to $0.81$--$0.86$, $0.90$--$0.95$, and
$0.69$--$0.78$, respectively, but increases the fixed evaluation cost
from six to ten continuations per boundary. We therefore cap boundary
evaluation at three PRE/POST pairs and use successive halving to
allocate this replication budget adaptively rather than uniformly.

\noindent\textbf{Successive-halving allocation.}
We next compare fixed three-pair evaluation with the three-round
successive-halving procedure used by \texttt{PAIR}. Using the
nine-draw selection as a higher-sample reference, fixed three-pair
evaluation attains $0.85$ recall and $0.64$ precision, whereas
successive halving attains $0.82$ recall and $0.66$ precision while
reducing the average evaluation from six to approximately $3.5$
continuations per boundary. Rule-support coverage is essentially
unchanged, and re-running prompt adaptation with the resulting evidence
preserves downstream performance. 
% Because the halving schedule was
% chosen after inspecting this diagnostic sweep, we treat these results
% as a retrospective robustness analysis rather than held-out validation.

\noindent\textbf{Selection thresholds.}
To isolate sensitivity to the selection rule from adaptive allocation,
we hold the continuation budget fixed at three PRE/POST pairs per
boundary and vary one threshold at a time around the default rule
$\widehat{H}_t\geq0.5$ or $\widehat{B}_t\geq5$. Under this fixed
three-pair diagnostic, $\widehat{H}_t$ changes in increments of $1/3$,
so the default hazard threshold requires an observed success-rate
difference of at least $2/3$.

\begin{figure}[h]
    \centering
    \includegraphics[width=0.82\linewidth]
    {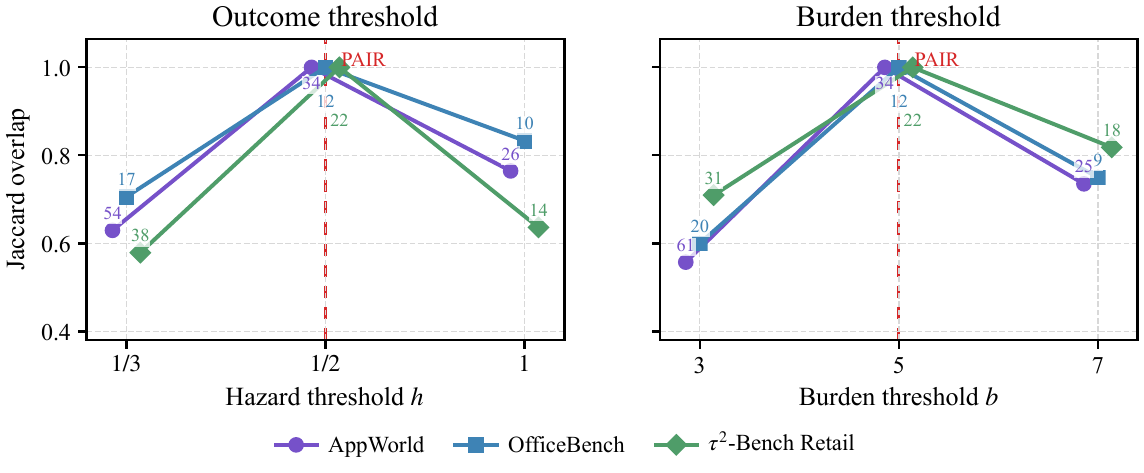}
    \caption{
        \textbf{Sensitivity to boundary-selection thresholds.}
        Holding the continuation budget fixed at three PRE/POST pairs,
        the left panel varies the outcome-hazard threshold while fixing
        the burden threshold at $5$; the right varies the burden
        threshold while fixing the hazard threshold at $0.5$.
        The vertical axis reports Jaccard overlap with the default
        selection, and labels beside the points report the number of
        selected boundaries. Red dashed lines mark the thresholds used
        by \texttt{PAIR}.
    }
    \label{fig:boundary-threshold-sensitivity}
    \vspace{-3mm}
\end{figure}

Figure~\ref{fig:boundary-threshold-sensitivity} shows the expected
trade-off between selectivity and coverage. Relaxing either threshold
adds lower-effect boundaries, whereas stricter thresholds retain a
smaller subset. Under the fixed three-pair diagnostic, the default rule
selects 34 AppWorld, 12 OfficeBench, and 22 $\tau^2$-Bench boundaries.
Across the one-at-a-time variations, the selected sets retain Jaccard
overlaps of $0.56$--$0.83$ with these default sets. We use the same
selection thresholds across all three benchmarks without
benchmark-specific tuning.

\subsection{Cost of Boundary-Level Adaptation}
\label{app:adaptation-cost}

Table~\ref{tab:adaptation-cost} reports the offline token cost of
\texttt{PAIR}. We account for initial trajectory collection, boundary
evaluation, prompt revision, and candidate selection. Boundary evaluation
uses the three-round successive-halving procedure described in Step~2:
every boundary first receives one PRE/POST continuation pair; after each
of the first two rounds, active boundaries are ranked by the normalized
harm score in Equation~\eqref{eq:halving-score}, and only the top half
receive another pair. Thus, each boundary receives at most three pairs.
Candidate selection uses five prompts and the 12 training tasks with the
most compression events, with one run per prompt--task pair. The
nine-draw diagnostic analysis, sensitivity experiments, and held-out test
evaluation are excluded.

\begin{table}[h]
\centering
\small
\caption{
Offline cost of boundary-level adaptation. Tokens include the visible
inputs and outputs of the downstream agent, compressor, and optimizer.
CF denotes PRE/POST counterfactual continuations allocated by three-round
successive halving. Total includes initial trajectory collection;
}
\label{tab:adaptation-cost}
\begin{tabular}{llrrr}
\toprule
Benchmark & Scope & Boundaries & CF rollouts & Total (M) \\
\midrule
AppWorld
& History-only       & 133 & 464   & 52.04 \\
& Prefix-conditioned & 198 & 694   & 85.40 \\

OfficeBench
& History-only       & 214 & 750   & 45.47 \\
& Prefix-conditioned & 298 & 1,042 & 90.79 \\

$\tau^2$-Bench Retail
& History-only       & 200 & 700   & 44.03 \\
& Prefix-conditioned & 204 & 714   & 39.59 \\
\bottomrule
\end{tabular}
\end{table}

Successive halving reduces the number of counterfactual continuations by
approximately 42\% relative to allocating three PRE and three POST
continuations uniformly to every boundary. Because retained boundaries
tend to require longer continuations, the corresponding reduction in
boundary-evaluation tokens is smaller but still substantial, ranging
from 30.2\% to 44.5\% across settings. Across the six adaptation
settings, the total offline token cost decreases from 511.24M to
357.32M tokens, a reduction of 30.1\%. Prompt revision itself remains
inexpensive, and the resulting total adaptation cost ranges from 39.6M
to 90.8M tokens. This cost is incurred once during offline adaptation;
after prompt selection, \texttt{PAIR} requires no additional model calls
beyond the standard deployment-time compression procedure.

\noindent\textbf{Why adaptive allocation?}
We also examined simpler ways to reduce boundary-evaluation cost.
Position-based truncation can miss later harmful compressions; retaining
only the first outcome-harm boundary can discard evidence supporting
distinct prompt revisions; and random or trajectory-level subsampling
loses coverage because useful boundary evidence is sparse. Successive
halving instead probes every realized compression once and reduces only
the amount of repeated sampling: additional continuation pairs are
concentrated on boundaries showing stronger provisional evidence of harm.
In retrospective nine-draw analyses, this reduces average evaluation from
six to approximately $3.5$ continuations per boundary, with only a small
change in harmful-boundary recall and essentially unchanged rule-support
coverage. Re-running prompt adaptation with the resulting evidence
preserves downstream performance.

\noindent\textbf{Comparison with ACON.}
PAIR incurs higher offline adaptation cost than ACON because it
explicitly evaluates individual compression boundaries rather than
relying only on trajectory-level feedback. Under the same
candidate-selection protocol, three-round successive halving requires
39.6--90.8M tokens across the six adaptation settings, compared with
11.8--27.9M for ACON-UT and 7.3--19.6M for ACON-UTCO.
Successive halving nevertheless substantially reduces the additional
attribution cost relative to uniform three-pair evaluation. Importantly,
this overhead is incurred only once during offline prompt adaptation;
after the prompt is selected, PAIR introduces no additional model calls
at deployment time.

\subsection{Cross-Agent Transfer}
\label{app:cross-agent-transfer}

We evaluate whether prompts adapted with GPT-5.6 Luna transfer to a
different model without further optimization. We apply the learned
prompts unchanged on AppWorld and OfficeBench, using MiniMax-M3 as both
the downstream agent and compressor.

\begin{figure*}[h]
    \centering
    \includegraphics[width=0.86\textwidth]
    {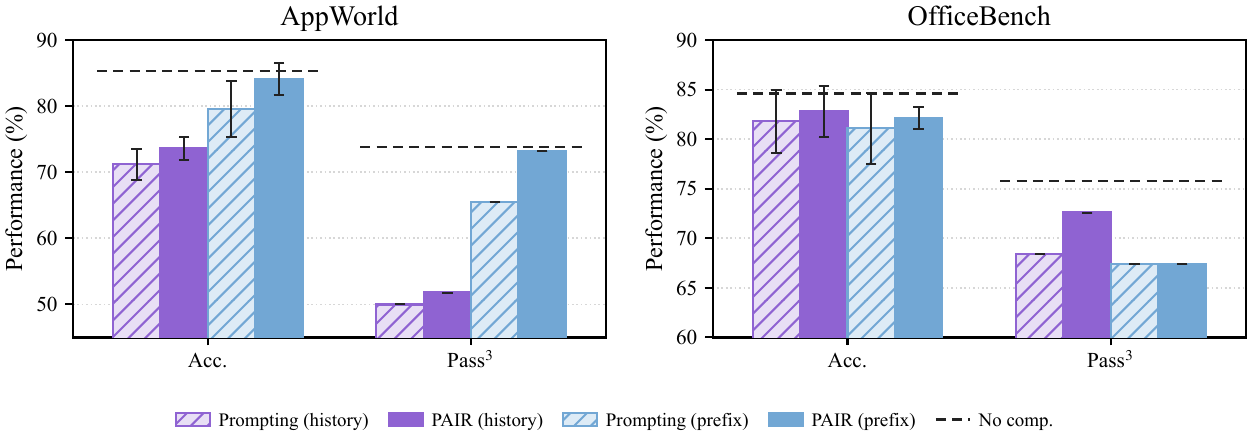}
    \caption{
        \textbf{Zero-shot transfer to MiniMax-M3.}
        Accuracy and Pass$^3$ over three runs after transferring prompts
        adapted with GPT-5.6 Luna. Both the agent and compressor use
        MiniMax-M3. Dashed lines denote no compression; error bars show
        the across-run standard deviation of accuracy.
    }
    \label{fig:cross-model-transfer}
    \vspace{-3mm}
\end{figure*}

As shown in Figure~\ref{fig:cross-model-transfer}, \texttt{PAIR}
improves upon its starting prompt after both the agent and compressor
are replaced. On AppWorld, history-only adaptation raises accuracy from
$71.2\%$ to $73.6\%$ and Pass$^3$ from $50.0\%$ to $51.8\%$.
Prefix-conditioned adaptation yields larger gains, increasing accuracy
from $79.6\%$ to $84.1\%$ and Pass$^3$ from $65.5\%$ to $73.2\%$,
approaching the no-compression results of $85.3\%$ and $73.8\%$.

The same prompts also transfer to OfficeBench. History-only
\texttt{PAIR} improves accuracy from $81.8\%$ to $82.8\%$ and Pass$^3$
from $68.4\%$ to $72.6\%$. Under prefix-conditioned compression,
accuracy increases from $81.1\%$ to $82.1\%$, while Pass$^3$ remains
$67.4\%$. Thus, the task-specific compression guidance learned with
GPT-5.6 Luna remains useful under a different agent--compressor pair,
although the magnitude of transfer depends on the benchmark and
compression scope.

\subsection{Paired Significance Tests Against ACON}
\label{app:paired}

Pass$^3$ is computed over three independent runs and changes by one
task whenever a single run flips, so aggregate differences of a few
points can lie within run-to-run noise. To assess whether \texttt{PAIR}'s
reliability gains are systematic, we report task-level paired statistics
against the best ACON variant in each of the six benchmark--scope
settings, where the best variant is the one with the higher Pass$^3$ in
Tables~\ref{tab:appworld_gpt_luna} and~\ref{tab:qa_officebench_main}.

\paragraph{Protocol.}
For each setting, let $S^{\mathrm{PAIR}}_{i,r}, S^{\mathrm{ACON}}_{i,r}
\in \{0,1\}$ denote success of task $i$ in run $r \in \{1,2,3\}$.
Because both methods are evaluated on the same task set, tasks serve as
the unit of pairing. We compute (i) paired bootstrap 95\% percentile
intervals for $\Delta\mathrm{Pass}^3$ and $\Delta\mathrm{Acc}$ by
resampling tasks with replacement 10{,}000 times and recomputing both
metrics for both methods on each resample, and (ii) an exact two-sided
sign test on the Pass$^3$ indicator, where a task counts as a
\emph{win} if \texttt{PAIR} solves it in all three runs and ACON does
not, and as a \emph{loss} in the reverse case; tasks on which both
methods agree are uninformative and excluded. Table~\ref{tab:paired}
reports the results.

\begin{table}[h!]
\centering
\caption{Task-level paired comparison of \texttt{PAIR} against the best
ACON variant in each setting. $\Delta$ is \texttt{PAIR} minus ACON in
percentage points with paired bootstrap 95\% intervals; win/lose counts
discordant tasks on the Pass$^3$ indicator. For each benchmark--scope
setting, $p$ is the two-sided exact sign test; the pooled win/lose count
is descriptive. Bold marks intervals that exclude zero.}
\label{tab:paired}
\resizebox{\textwidth}{!}{%
\begin{tabular}{llrrrrr}
\toprule
Benchmark & Scope & vs. & $n$ & $\Delta$Pass$^3$ [95\% CI] & $\Delta$Acc [95\% CI] & win/lose ($p$) \\
\midrule
AppWorld    & history & ACON-UT   & 168 & $\mathbf{+8.9\ [+3.0,\,+15.5]}$ & $\mathbf{+3.6\ [+0.8,\,+6.3]}$ & 22/7 (0.008) \\
AppWorld    & prefix  & ACON-UT   & 168 & $\mathbf{+7.8\ [+1.8,\,+10.7]}$ & $\mathbf{+3.2\ [+0.4,\,+6.2]}$ & 21/8 (0.024) \\
OfficeBench & history & ACON-UT   &  95 & $+2.1\ [-5.3,\,+9.4]$ & $+3.5\ [-0.4,\,+7.4]$ & 7/5 (0.774) \\
OfficeBench & prefix  & ACON-UTCO &  95 & $+7.4\ [-1.1,\,+16.8]$ & $+2.5\ [-1.8,\,+7.0]$ & 13/6 (0.167) \\
$\tau^2$-Retail & history & ACON-UTCO & 40 & $+2.5\ [-15.0,\,+20.0]$ & $+0.0\ [-10.0,\,+9.2]$ & 7/6 (1.000) \\
$\tau^2$-Retail & prefix  & ACON-UT   & 40 & $+7.5\ [-5.0,\,+20.0]$ & $+4.2\ [-5.8,\,+15.0]$ & 5/2 (0.453) \\
\midrule
Pooled & & & 606 & & & 75/34 \\
\bottomrule
\end{tabular}%
}
\end{table}

\paragraph{Results.}
All six point estimates of $\Delta\mathrm{Pass}^3$ are positive, and
\texttt{PAIR} wins more discordant tasks than it loses in every
benchmark--scope setting. Taken individually, the Pass$^3$ advantage is
significant at the 95\% level in both AppWorld settings. The remaining
intervals include zero, reflecting the limited resolution of Pass$^3$
over three runs on 40--95 tasks rather than a reversal in direction:
on $\tau^2$-Bench Retail, for example, a single task corresponds to
2.5 percentage points. Pooling the task--scope comparisons
descriptively yields 75 wins against 34 losses. Because the two
compression scopes reuse the same underlying tasks, we do not treat
this pooled count as an independent-sample significance test. Overall,
the paired results show a consistent directional advantage across
benchmarks and compression scopes rather than an effect driven by a
single favorable setting. We therefore report the per-setting effect
sizes together with their uncertainty intervals and paired tests.

\subsection{Boundary-Level Case Studies}
\label{app:case-study}

We examine how boundary-level evidence is translated into compression
prompt revisions on AppWorld. We focus on prefix-conditioned
compression, for which the complete chain from the initial prompt to
the selected \texttt{PAIR} prompt is available. At each studied
boundary, we restore the same environment state and recompress the same
consumed history using the initial prompt $P_0$, the adapted prompt
$P_{\mathrm{PAIR}}$, and, for the outcome cases, the matched
\texttt{ACON-UT} prompt $P_{\mathrm{ACON}}$. Each condition uses three
independent compression samples and continuations, while the agent,
tools, decoding configuration, and environment state remain fixed.
These targeted replays are conducted only for analysis and do not
affect prompt selection.

We consider a case closed when the original defect recurs after
recompression with $P_0$, disappears with $P_{\mathrm{PAIR}}$, and the
adapted summaries exhibit the intended prompt change. The control is
necessary because the boundary effect concerns a realized summary: a
harmful summary may be sampled from a prompt without the same defect
recurring in every subsequent sample. Table~\ref{tab:case-study-summary}
summarizes the two primary cases.

% Preamble:
% \usepackage{tabularx}
% \usepackage{array}

\begin{table*}[h!]
\centering
\small
\caption{
\textbf{Boundary-level case studies on AppWorld.}
The control recompresses the same boundary using the initial prompt,
whereas the adapted condition uses the selected \texttt{PAIR} prompt.
Success counts are measured over three continuations. PRE denotes the
pre-compression context.
}
\label{tab:case-study-summary}

{
\setlength{\tabcolsep}{4pt}
\renewcommand{\arraystretch}{1.08}
\begin{tabularx}{\linewidth}{
    @{}
    >{\raggedright\arraybackslash}p{0.11\linewidth}
    >{\raggedright\arraybackslash}p{0.15\linewidth}
    >{\raggedright\arraybackslash}X
    >{\raggedright\arraybackslash}X
    >{\raggedright\arraybackslash}p{0.16\linewidth}
    @{}
}
\toprule
Case & Boundary effect & Compression error & Prompt revision &
Control $\rightarrow$ \texttt{PAIR} \\
\midrule

Outcome hazard &
PRE $3/3$, POST $1/3$;
$\widehat H_t=0.67$ &
The summary drops the coworker restriction and presents the total
over all received payments as the answer. &
Preserve task predicates and relationship filters; distinguish
observations from unverified inferences. &
$0/3 \rightarrow 3/3$ success \\

\addlinespace[2pt]

Interaction burden &
PRE and POST both $3/3$;
$\widehat B_t=5.0$ &
Previously read API specifications are reduced to prose, causing
documentation and authentication to be repeated. &
Preserve executable API signatures and completed state; do not reopen
completed documentation or authentication. &
$14.7 \rightarrow 10.7$ mean steps \\

\bottomrule
\end{tabularx}
}
\end{table*}

\noindent\textbf{Case A: Preserving a task-defining predicate.}
The task asks for the amount received from the user's coworkers on
Venmo after a specified date. Before compression, the agent has
retrieved 36 received transactions totaling 1,676, but has not yet
determined which senders are coworkers. The summary produced by
$P_0$ removes this unresolved restriction and treats the unrestricted
total as the final answer:

\begin{promptbox}{Harmful summary produced by the initial prompt}
- [x] Retrieved all received Venmo transactions from
  2023-02-01 onward across all pages: 36 transactions.

- Filter by direction='received' and
  min_created_at='2023-02-01': Matches the request for money
  received since February 1.

- Use total 1676 as the answer: Sum of the 36 returned
  transaction amounts.
\end{promptbox}

The three PRE continuations query the phone contacts for the coworker
relationship, retain only matching Venmo senders, and submit the
correct answer, 786. Two of the three original POST continuations
instead submit 1,676 immediately, while the remaining continuation
recovers the missing predicate through further interaction. Therefore,
$\widehat V_t(\mathrm{PRE})=1$,
$\widehat V_t(\mathrm{POST})=1/3$, and
$\widehat H_t=0.67$. The optimizer identifies the dropped relationship
constraint as the primary cause of the outcome difference.

The resulting revision changes how the compressor records the task
objective and distinguishes established facts from unresolved
inferences:

\begin{promptbox}{Relevant PAIR prompt revision}
Preserve every task predicate, scope boundary, relationship filter,
temporal condition, required mutation, and completion condition
exactly. Never replace a restricted set with a broader set such as all
received records, all pending records, or all search results.

Record only decisions supported by the task, operating rules,
documented API behavior, or observed history. Label inferences as
tentative.
\end{promptbox}

When the same boundary is recompressed with $P_0$, all three new
summaries again direct the agent to answer 1,676, and all three
continuations fail. In contrast, $P_{\mathrm{PAIR}}$ preserves both the
observed total and its unresolved relationship to the requested
coworker-only quantity:

\begin{promptbox}{Summary produced by the adapted PAIR prompt}
- [x] The sum of all 36 received transactions is 1676.
  This is not yet confirmed to be the coworker-only total.

- [ ] Determine which Venmo senders are Paul's coworkers using the
  phone contacts relationship data.

- Coworker scope: Tentative; treating all 36 received transactions
  as coworker payments is not explicitly verified.
\end{promptbox}

All three adapted continuations first query the contact relationship
and then submit 786, yielding $3/3$ success. The fix requires more
steps than the incorrect immediate submission, but only because the
control terminates early with the wrong answer. Relative to PRE, the
adapted continuations add approximately three steps while restoring
full success. At the same boundary, \texttt{ACON-UT} succeeds in two
of three draws: two summaries recover the coworker restriction, while
one again instructs the agent to use 1,676.

A second boundary exhibits the same mechanism. For a task restricted
to requests from coworkers and friends, the original summary broadens
the target to every pending request. The original POST continuations
fail $3/3$. Under recompression, $P_0$ succeeds only $1/3$, whereas
$P_{\mathrm{PAIR}}$ explicitly marks the relationship classification
as unresolved and succeeds $3/3$. These cases show that the revision
addresses a recurring loss of task-defining predicates rather than a
single surface form.

\noindent\textbf{Case B: Preserving usable execution state.}
The second case separates outcome preservation from interaction
burden. The task requires reading an internet-bill receipt, identifying
the user's roommates, and sending equal Venmo requests. Before
compression, the agent has already read the login and contact-search
specifications. All PRE and POST continuations eventually succeed, but
the POST continuations require five additional steps on average:
$\widehat B_t=14.0-9.0=5.0$.

The initial summary records that the specifications were read but
reduces them to generic prose and distributes authentication across
several future steps:

\begin{promptbox}{Burden-inducing summary produced by the initial prompt}
- [x] Retrieved specifications for file_system.login.

- [x] Retrieved specifications for phone.login.

2. Get the file-system account credentials and call
   apis.file_system.login.

5. Log into phone using the user's phone number and
   supervisor-provided password.

7. Retrieve the Venmo account credentials and log in if required.

File-system login requires a username and password and returns an
access token. Phone login requires the account phone number and a
password.
\end{promptbox}

Each PRE continuation obtains the required credentials, authenticates
all three applications in one interaction step, and reads only the
previously unseen Venmo request documentation. The POST continuations
instead reopen the file-system and phone login documentation,
authenticate the applications separately, and perform four to five
documentation calls. This divergence does not change task success, but
it consumes execution budget that would otherwise remain available for
recovery.

The initial boundary diagnosis was only partially correct. Its
top-ranked hypothesis concerned an added post-request verification
step, but no POST continuation performed such verification. The traces
instead show repeated documentation and fragmented authentication.
When evidence from this boundary is aggregated with related burden
boundaries, the selected prompt introduces the following
compressor-facing guidance:

\begin{promptbox}{Relevant PAIR prompt revision}
Preserve actionable completed state, including exact usable tokens,
credentials, established variables, API signatures, and already-read
documentation. Do not reduce usable state to a generic statement that
it exists.

Give the smallest ordered continuation justified by the established
state. Do not reopen completed documentation, authentication, target
discovery, or data retrieval.
\end{promptbox}

Under the same-prompt replay, the burden recurs: the three
continuations take 17, 14, and 13 steps, for a mean of 14.7. The
adapted summaries instead retain executable signatures and consolidate
authentication:

\begin{promptbox}{Execution state preserved by the adapted prompt}
apis.file_system.login(
username=<email>, password=<password>)

apis.phone.login(
username=<phone_number>, password=<password>)

apis.phone.search_contacts(
access_token=<token>, query="", relationship=None,
page_index=0, page_limit=5)

1. Retrieve the supervisor-provided passwords and use the relevant
   credentials to log in to file_system, phone, and venmo.
\end{promptbox}

All three adapted continuations authenticate the three applications in
one step and read only the previously unseen Venmo documentation. Mean
continuation length falls from 14.7 to 10.7 steps while success remains
$3/3$. The adapted execution does not fully recover the PRE mean of
9.0 steps, but removes most of the reproducible burden introduced by
the initial prompt.

\noindent\textbf{Why boundary-level evidence yields a different
adaptation from ACON.}
The advantage in these cases is not that \texttt{PAIR} alone can
recognize the surface error. Under the controlled comparison,
\texttt{ACON-UT} and \texttt{PAIR} start from the same OpenClaw prompt
and analyze the same seven failed training tasks. ACON's trajectory
analyses also identify the missing coworker filter and omitted
transaction identifiers. The key difference lies in how the evidence
is attributed and converted into a compression policy.

ACON compares a successful full-context trajectory with a failed
compressed trajectory. Its analysis template requests broad
remediation strategies, including guardrails, verification, caching,
early-exit heuristics, and loop detection. Because the terminal
contrast contains all actions and compressions preceding the outcome,
a detected compression error can be translated into instructions for
how the downstream agent should execute:

\begin{promptbox}{Representative agent-facing guidance introduced by ACON}
Before any irreversible action, require one compact guardrail check
that the target set, scope, and identifiers are complete.

After mutations, require verification that the intended state change
has occurred.

Maintain a State Table and a separate set of Compression Rules in the
summary.
\end{promptbox}

Such guidance may prevent some errors, but it couples the compression
prompt with a new execution policy. Verification, blocking, and
recovery behavior are delegated to the downstream agent even when the
underlying compression failure is simply that a predicate or usable
state was not preserved.

In contrast, the paired continuations in \texttt{PAIR} hold the
environment state and downstream agent fixed at one boundary. The
observed difference is therefore localized to one context replacement.
Outcome hazard identifies information whose loss changes task
completion, while interaction burden identifies information whose loss
induces recovery without changing the terminal result. The optimizer
is then restricted to revising what the compressor records:

\begin{promptbox}{Compressor-facing guidance introduced by PAIR}
Preserve every task predicate, scope boundary, relationship filter,
temporal condition, required mutation, and completion condition
exactly.

Preserve actionable completed state, including exact usable tokens,
credentials, established variables, API signatures, and already-read
documentation.

\end{promptbox}

This boundary-level construction provides two forms of decoupling.
First, it separates compression-induced changes from errors already
present in the trajectory: PRE and POST begin from the same environment
state and differ only in the supplied context. Second, it separates
the compressor's responsibility from the agent's policy. The resulting
rules specify which state must survive compression rather than adding
new verification procedures, blockers, or action-selection heuristics
to the summary.

The downstream traces are consistent with this distinction. On the
AppWorld test-normal split, ACON summaries preserve explicit access
tokens in none of the observed summaries, compared with $69\%$ for
\texttt{PAIR}; executable API calls with keyword arguments appear in
$48\%$ and $80\%$, respectively. After a compression, an API
documentation lookup is the first action in $38\%$ of ACON
continuations and $22\%$ of \texttt{PAIR} continuations. ACON also
requires 6.05 steps on average to reach the next mutation or
completion, compared with 4.74 for \texttt{PAIR}, and incurs 7.9
versus 5.9 documentation calls per task. These measurements are
behavioral traces consistent with the prompt distinction; they do not
attribute every held-out action to an individual prompt line.

End-to-end results exhibit the same pattern. Relative to the matched
\texttt{ACON-UT} prompt, prefix-conditioned \texttt{PAIR} improves mean
success by $3.2\%$, with a paired $95\%$ interval of
$[0.4\%,6.0\%]$, reduces mean execution length by 2.15 steps, with an
interval of $[-2.88,-1.45]$, and reduces total token use by $12.6\%$,
with an interval of $[-17.2\%,-7.7\%]$. The Pass$^3$ difference is
positive but less certain: $+3.6\%$ with an interval of
$[-1.8\%,8.9\%]$. The evidence therefore supports a clearer conclusion
about average success and execution efficiency than about the
magnitude of the reliability difference in this individual
comparison.

\noindent\textbf{Negative and non-reproducible cases.}
We replay all seven outcome boundaries cited by the selected error
rules and all six rule-cited burden boundaries with equal PRE and POST
success, rather than selecting examples only after replay. Three
realized outcome defects do not recur under fresh compression samples
from the initial prompt. Their original summaries are harmful, but
the defects are not stable across subsequent samples from the same
compression policy. One cited outcome boundary remains unsuccessful
under $P_{\mathrm{PAIR}}$, and two documentation-related burden
boundaries are not shortened by the adapted prompt. Overall, three
outcome boundaries and one burden boundary satisfy the complete replay
criterion. Cases A and B are selected because they provide the
clearest alignment between the observed divergence, the resulting
prompt rule, and replayed behavior.

These cases illustrate both the value and the limit of boundary-level
adaptation. A localized counterfactual effect provides cleaner evidence
than a terminal trajectory contrast, but a finite number of stochastic
continuations cannot guarantee that every diagnosis generalizes at the
prompt level. End-to-end selection therefore remains necessary to
determine whether the aggregated revisions improve the global
compression policy.

\section{Boundary-Verification Algorithm}
Algorithm~\ref{alg:boundary-halving} gives the implementation of the
successive-halving procedure used in Step~2. Every realized compression
boundary is evaluated with an initial PRE/POST continuation pair.
Additional pairs are allocated only to the higher-scoring half of the
active boundaries according to the normalized harm score in
Equation~\eqref{eq:halving-score}. After three rounds, surviving
boundaries are retained if their estimated outcome hazard or interaction
burden exceeds the corresponding selection threshold.
\begin{algorithm}[h]
\caption{Successive-Halving Boundary Verification}
\label{alg:boundary-halving}
\begin{algorithmic}[1]
\Require Realized boundaries $\mathcal A$; thresholds $\tau_H,\tau_B$
\State $\mathcal A^{(1)} \gets \mathcal A$

\For{$r=1,2,3$}
    \For{each $t\in\mathcal A^{(r)}$}
        \State Run one additional PRE/POST continuation pair
        \State Update $\widehat H_t^{(r)}$ and $\widehat B_t^{(r)}$
    \EndFor

\If{$r<3$}
    \State Compute
    $s_t^{(r)}
    \gets
    \max\!\left\{
    \widehat H_t^{(r)}/\tau_H,\,
    \widehat B_t^{(r)}/\tau_B
    \right\}$
    for all $t\in\mathcal A^{(r)}$
    \State $\mathcal A^{(r+1)}
    \gets \operatorname{TopHalf}
    \bigl(\mathcal A^{(r)};s^{(r)}\bigr)$
\EndIf
\EndFor

\State $\mathcal V \gets
\left\{
t\in\mathcal A^{(3)}:
\widehat H_t^{(3)}\geq\tau_H
\ \lor\
\widehat B_t^{(3)}\geq\tau_B
\right\}$
\State \Return $\mathcal V$ and the associated paired traces
\end{algorithmic}
\end{algorithm}

\newpage
\section{Details of Baselines}\label{app:baseline_details}

This section describes the compression baselines used in our experiments.
Unless otherwise specified, all conditions are evaluated under the same
frozen-agent protocol: the downstream agent model, tool-use prompt, decoding
configuration, tool APIs, output parser, and execution environment are fixed.
Only the context supplied to the agent is changed. This protocol isolates the
effect of context representation from changes in the downstream agent policy.

\noindent\textbf{Shared Evaluation Protocol.}
All compression and truncation baselines use the same recurrent-compression
trigger and preserve the most recent interaction turn verbatim. The
full-context reference bypasses compression and retains the complete interaction
history. Across all conditions, we keep the downstream tool-use prompt, tool
descriptions, output-format instructions, few-shot examples, decoding
configuration, parser, and execution environment fixed. Replacement contexts
are inserted at the same continuation point, so differences in downstream
behavior arise from the supplied context representation rather than from a
changed agent policy or interface.

Because prompt-defined baselines can be sensitive to small wording changes, we
freeze all prompt templates before evaluation and record hashes of the rendered
prompts used in each run. For ACON, the source of truth is the original
Microsoft repository and commit specified below. For LLMLingua-2, the source
of truth is the official Microsoft implementation and released compression
model.

\subsection{Full-Context Reference}
\label{app:baseline-full-context}

\noindent\textbf{No compression.}
The agent receives the full uncompressed interaction history. This condition
serves as the full-context reference for behavior preservation and as the
reference point for token-cost measurements. It is not a compressor and does
not separately consume the task instruction, which is already included in the
agent context.

\subsection{Token Pruning and Truncation Baselines}
\label{app:baseline-llmlingua2}

This group contains two non-generative baselines. Rather than producing a new
free-form summary, they retain selected portions of the original interaction
history. Both use the same compression trigger, downstream context slot, and
recent-turn preservation policy as the generative baselines. compression is
triggered when the compressible history exceeds the context budget, while the
most recent interaction turn remains verbatim.

\noindent\textbf{FIFO.}
FIFO is a recency-based sliding-window control. When the rendered compressible
history exceeds the context budget, complete turns are discarded from the
front, oldest first, until the history fits. The system prompt, original task
instruction, and most recent interaction turn are always retained. FIFO
therefore isolates how much behavior can be preserved through recent
action--observation continuity alone, without learned salience estimation or
generated summary text.

\noindent\textbf{LLMLingua-2.}
LLMLingua-2 formulates prompt compression as token classification and distills
a smaller compressor for efficient and faithful extractive compression
\citep{pan2024llmlingua}. We apply it task-agnostically to the compressible
interaction history using the released
\texttt{microsoft/llmlingua-2-xlm-roberta-large-meetingbank} model. When
compression is triggered, LLMLingua-2 selects tokens from the existing history
up to the target budget. The resulting extractive context replaces the older
turns, while the most recent turn remains verbatim. This baseline tests whether
token-level salience alone preserves the execution state required for future
agent actions.

\noindent\textbf{Implementation.}
We use the official Microsoft LLMLingua implementation.\footnote{
\url{https://github.com/microsoft/LLMLingua}, version \texttt{0.2.2}
(release tag \texttt{v0.2.2}, commit
\texttt{a411a3fa61df74411157b2512b592d5357bd8f17}).}
For LLMLingua-2, the target token count is set to the budget allocated to the
compressible history, excluding the most recent turn that is retained
verbatim. Its compressed output is inserted into the same downstream context
slot used by the other compression baselines.

\subsection{Structured-Summary Compression Baselines}
\label{app:baseline-structured-summaries}

Our two structured-summary baselines are adapted from compression modules in
open-source agent frameworks. We preserve their original summary schemas and
prompt text while integrating them into the same recurrent-compression harness.
Unlike the token-dropping LLMLingua-2 baseline and the recency-based FIFO
control, both invoke an auxiliary LLM to rewrite the compressible history into
a structured Markdown checkpoint after the context exceeds the token budget.
They preserve the most recent interaction turn and support iterative updates
that fold new turns into the previous checkpoint.

\paragraph{Prompting (OpenClaw compaction).}
The OpenClaw baseline\footnote{Adapted from the OpenClaw agent-core harness:
\url{https://github.com/openclaw/openclaw/blob/0e7b5c34292cc28707a0e5a0b730cff295ef0f8a/packages/agent-core/src/harness/compaction/compaction.ts}
(commit \texttt{0e7b5c34292cc28707a0e5a0b730cff295ef0f8a}).}
maintains an approximate recent-token budget, cuts the compressible history at
a turn boundary, and summarizes the remainder with the prompts below. The
first compression uses the checkpoint prompt. Subsequent compressions use the
update prompt, which folds new turns into the previous summary.

\begin{promptbox}{OpenClaw summarization system prompt}
You are a context summarization assistant. Your task is to read a conversation between a user and an AI coding assistant, then produce a structured summary following the exact format specified.

Do NOT continue the conversation. Do NOT respond to any questions in the conversation. ONLY output the structured summary.
\end{promptbox}

\begin{promptbox}{OpenClaw first-compaction prompt}
The messages above are a conversation to summarize. Create a structured context checkpoint summary that another LLM will use to continue the work.

Use this EXACT format:

## Goal
[What is the user trying to accomplish? Can be multiple items if the session covers different tasks.]

## Constraints & Preferences
- [Any constraints, preferences, or requirements mentioned by user]
- [Or "(none)" if none were mentioned]

## Progress
### Done
- [x] [Completed tasks/changes]

### In Progress
- [ ] [Current work]

### Blocked
- [Issues preventing progress, if any]

## Key Decisions
- **[Decision]**: [Brief rationale]

## Next Steps
1. [Ordered list of what should happen next]

## Critical Context
- [Any data, examples, or references needed to continue]
- [Or "(none)" if not applicable]

Keep each section concise. Preserve exact file paths, function names, and error messages.
\end{promptbox}

\begin{promptbox}{OpenClaw iterative-update prompt}
The messages above are NEW conversation messages to incorporate into the existing summary provided in <previous-summary> tags.

Update the existing structured summary with new information. RULES:
- PRESERVE all existing information from the previous summary
- ADD new progress, decisions, and context from the new messages
- UPDATE the Progress section: move items from "In Progress" to "Done" when completed
- UPDATE "Next Steps" based on what was accomplished
- PRESERVE exact file paths, function names, and error messages
- If something is no longer relevant, you may remove it

Use this EXACT format:

## Goal
[What is the user trying to accomplish? Can be multiple items if the session covers different tasks.]

## Constraints & Preferences
- [Any constraints, preferences, or requirements mentioned by user]
- [Or "(none)" if none were mentioned]

## Progress
### Done
- [x] [Completed tasks/changes]

### In Progress
- [ ] [Current work]

### Blocked
- [Issues preventing progress, if any]

## Key Decisions
- **[Decision]**: [Brief rationale]

## Next Steps
1. [Ordered list of what should happen next]

## Critical Context
- [Any data, examples, or references needed to continue]
- [Or "(none)" if not applicable]

Keep each section concise. Preserve exact file paths, function names, and error messages.
\end{promptbox}

\subsection{ACON Prompt Baselines}
\label{app:baseline-acon}

We compare against two prompt-optimization methods from
ACON~\citep{kang2025acon}. Both adapt the natural-language compression
guideline while keeping the downstream agent and compressor model fixed.

\noindent\textbf{\textsc{ACON-UT}.}
The utility-maximization stage identifies tasks that succeed with full
context but fail under compression. An optimizer LLM compares the
corresponding trajectories, infers which task-relevant information was
lost or distorted, and revises the compression guideline to improve
task success.

\noindent\textbf{\textsc{ACON-UTCO}.}
This method applies compression maximization after utility maximization.
The optimizer analyzes successful compressed trajectories to identify
redundant content and further revises the guideline to produce shorter
contexts while preserving task utility.

\noindent\textbf{Implementation.}
We use the official prompt-optimization implementation from the
Microsoft ACON repository.\footnote{
\url{https://github.com/microsoft/acon}. We use commit
\texttt{d63f9ae18959dc7215ff62899c94c5e8c56847ae}.}
Both methods start from the same OpenClaw-style compression prompt used
by our \textsc{Prompting} baseline. \textsc{ACON-UT} applies only
utility maximization, whereas \textsc{ACON-UTCO} subsequently applies
compression maximization. All other agent, compressor, and execution
settings remain fixed.

\section{Prompts and Qualitative Analysis}
ACON-UT and \texttt{PAIR} start from the same OpenClaw compression
prompt and use the same AppWorld adaptation trajectories. They differ,
however, in both the evidence used for adaptation and the space over
which the prompt may be revised. ACON contrasts successful
full-context and failed compressed trajectories and may freely rewrite
the complete compression framework. \texttt{PAIR} uses localized
counterfactual evidence at harmful compression boundaries while keeping
the OpenClaw section structure fixed; it adapts only the instructions
within that deployment-compatible template. We present the resulting
prompts in full before analyzing these differences.

\begin{promptbox}{History-only {ACON-UT} Prompt Optimized on AppWorld}
<conversation>
{{ history }}
</conversation>

<previous-summary>
{{ prev_summary }}
</previous-summary>

[INFORMATION SOURCE]
The messages in <conversation> are new execution evidence. The content in <previous-summary> is prior persisted state. Reconcile both into one compact, factually continuous summary.

Rules for updating:
- Treat observed tool calls and successful responses as authoritative over plans, prose checklists, or inferred status.
- Preserve completed actions as completed; never label a successful call as blocked, unknown, or pending.
- Distinguish every action as exactly one of: Done, In Progress, Blocked, or Failed. A task is not pending if its API call and successful result are present.
- Preserve immutable outputs, exact identifiers, dates, amounts, names, emails, file paths, endpoint/function names, parameters, error messages, accepted answers, completion calls, and completion confirmations when they affect continuation.
- Preserve deterministic derived results only with their source or derivation: eligibility criteria, target ID sets, exclusions, totals, mappings, and verification evidence.
- Do not infer facts from missing variables, session boundaries, stale summaries, external dates, or intended plans.
- If prior prose conflicts with concrete tool output, replace the prose with the tool-supported fact and note the correction only when useful.
- Python/runtime variables do not survive a session reset. Never assume a prior-session variable exists unless it is listed below and recreated in the new session.
- For authentication, record app, username/account identifier, authentication status, canonical token variable or safe token handle, token-validity assumption, and the exact re-login/bootstrap instruction needed if the token cannot persist. Do not print raw passwords or raw tokens.
- Preserve all state required to continue without rediscovery: IDs, page/cursor positions, page size, retrieved records or compact exact lists, mutation results, request/transaction IDs, contact/email mappings, and verification status.
- Record pagination completeness: starting page/cursor, page size, pages processed, termination condition, and whether the resulting set is complete. An empty page proves completion only when the required initial page and preceding pages were queried.
- Preserve endpoint contracts already established, including required parameters and constraints. Do not recommend repeated documentation, credential lookup, variable probing, or data retrieval unless the summary shows uncertainty or an observed failure.
- Preserve exact target sets and exclusion sets before mutations. Do not broaden, recompute, or mutate a set without reconciling it against the persisted evidence.
- Record postconditions and verification results for every externally visible or destructive action. Do not claim completion until the required result set has been processed and verified.
- If a valid answer or completion result already exists, preserve it as terminal and instruct the next session to stop exploration and retain or submit that result.
- Use early-exit conditions when supported evidence is sufficient, such as exhausted pagination, an empty subsequent page, an out-of-range feed page, an unsupported artifact format, or a validated fallback answer.
- Keep the summary compact: prefer tables, exact lists, and short bullets; omit conversational narrative, duplicated facts, speculative alternatives, and API output that is not needed for continuation.
[/INFORMATION SOURCE]

Create the updated summary using exactly this structure:

## Goal
- State the current task and any explicit scope or acceptance criteria.
- Preserve the original goal unless the new evidence definitively changed it.

## Execution Ledger
### Done
- [x] Completed action --- include the decisive tool result or output when needed for continuity.

### In Progress
- [ ] Only actions not yet completed and currently actionable.

### Blocked or Failed
- Record only evidence-based blockers or failures, including the exact error and the condition required to recover.
- Do not place unavailable-but-optional inputs here if a validated answer or fallback is sufficient.

### Terminal Completion
- `status`: `not_completed` or `completed`
- `accepted_answer`: [value, if any]
- `completion_call`: [exact function and arguments, if observed]
- `completion_confirmation`: [exact confirmation, if observed]
- `next_action`: [hard stop if completed; otherwise the single highest-priority action]

## State Table
| Category | Canonical name | Value or status | Evidence / continuation instruction |
|---|---|---|---|
| Authentication | [app and token variable/handle] | [authenticated, absent, expired, or unknown] | [safe re-login/bootstrap instruction and validity assumption] |
| Entities and mappings | [IDs, names, emails, relationships] | [compact exact values] | [source or derivation] |
| Collections and pagination | [collection/endpoint] | [IDs or records, page/cursor, page size] | [completeness and termination condition] |
| Eligibility and exclusions | [criterion] | [exact eligible set and excluded near-misses] | [deterministic rule/evidence] |
| Mutations | [action and target IDs] | [processed, succeeded, failed, or not attempted] | [per-item or aggregate result] |
| Verification | [postcondition] | [verified, failed, or not run] | [query/result and coverage] |
| Cached API knowledge | [app/endpoint] | [required parameters and constraints] | [reuse unless schema uncertainty causes failure] |
| Other critical state | [canonical variable] | [value/status] | [reconstruction instruction] |

Include only applicable rows, but do not omit any state required to resume safely.

## Key Decisions
- **[Decision]**: [brief evidence-based rationale]
- Preserve prior decisions that remain valid; revise contradictory decisions rather than duplicating them.

## Next Steps
1. [The single highest-priority actionable step, or the hard stop if terminal.]
2. [Only additional necessary steps, ordered by dependency.]
- Do not list documentation or rediscovery steps when the state table already supplies the required information.

## Continuation Guards
- Do not reference prior-session variables unless they appear in the State Table and are recreated or re-authenticated in the new session.
- Before any mutation, reconcile the current target set, exclusions, authentication, required parameters, and idempotency status.
- Before completion, verify that every qualifying target was processed and that the verification set equals the expected set.
- If later evidence changes a contact set, transaction set, pagination result, date, eligibility result, or completion status, invalidate the affected conclusion and recompute it from preserved source records.
- Execute the highest-priority Next Step before exploratory calls.
- Keep the entire summary concise while retaining all state needed for factual continuity.
\end{promptbox}

\begin{promptbox}{History-only \texttt{PAIR} Prompt Optimized on AppWorld}
<conversation>
{{ history }}
</conversation>

<previous-summary>
{{ prev_summary }}
</previous-summary>

The messages above are NEW conversation messages to incorporate into the existing summary provided in <previous-summary> tags.

Update the existing structured summary with new information.

RULES:
- PRESERVE all still-valid existing information from the previous summary, especially the exact user goal, scope, constraints, settled results, completed pages, identifiers, filters, and actionable continuation state.
- ADD only facts supported by the new messages. Do not invent a new task, constraint, result, API, capability, or pending action from documentation or credentials alone.
- UPDATE the Progress section using the newest settled state. Move an item from "In Progress" to "Done" only when the new messages show it completed; never leave a completed lookup or page described as unchecked.
- Preserve the distinction between the user's qualifying subset and broader aggregates, and between final task-specific results and intermediate subtotals, page totals, placeholders, or symbolic expressions.
- Retain exact file paths, function names, identifiers, query parameters, page indexes, required wording, and error messages when they remain relevant.
- Record authentication or reusable session availability when established, without copying secret passwords or token values. Do not make a completed login the next step unless the history shows it failed, expired, or otherwise must be retried.
- Remove or revise information only when the new messages establish that it is obsolete, contradicted, or no longer relevant; do not discard settled pagination, scope, or final-submission facts.
- Keep Next Steps limited to genuinely unresolved work and avoid speculative searches, repeated documentation reads, or re-investigation of settled facts.

Use this EXACT format:

## Goal
[Preserve the user's exact objective and scope. Do not broaden it or replace it with a task inferred from exploratory actions.]

## Constraints & Preferences
- [Preserve existing constraints and add only newly evidenced ones. Keep inclusion/exclusion criteria and required wording explicit.]

## Progress
### Done
- [x] [Include all previously completed items that remain valid and newly completed items. Include every completed page or batch relevant to scope.]

### In Progress
- [ ] [List only genuinely unresolved work after incorporating the new messages. Do not repeat completed actions.]

### Blocked
- [Current evidenced blockers only; preserve exact errors and unavailable data, and remove resolved blockers.]

## Key Decisions
- **[Decision]**: [Preserve all valid prior decisions and add newly evidenced decisions. Keep qualifying results separate from aggregates and intermediate values separate from the settled final result.]

## Next Steps
1. [Update only for genuinely unresolved work. Use the shortest actionable sequence and reuse established APIs, identifiers, filters, page indexes, and available session state.]

## Critical Context
- [Preserve important exact paths, function names, identifiers, query parameters, completed pages, outputs, errors, settled results, and final-submission requirements. Never replace a concrete value with a literal placeholder or unsupported expression, and never expose secret credentials.]
- [Or "(none)" if not applicable]

Keep each section concise while preserving all settled scope and continuation-critical state. Do not invent facts, goals, actions, results, identifiers, APIs, or capabilities. Preserve exact file paths, function names, query parameters, identifiers, and error messages.
\end{promptbox}

\noindent\textbf{Observed interface differences do not explain
planning-oriented adaptation.}
The original ACON-UT and \texttt{PAIR} prompts differ visibly in their
checkpoint interfaces. \texttt{PAIR} retains the OpenClaw sections,
\emph{Goal}, \emph{Constraints \& Preferences}, \emph{Progress},
\emph{Key Decisions}, \emph{Next Steps}, and
\emph{Critical Context}, and revises the instructions governing what
each section should preserve. The original ACON proposer imposes no
structural constraint, and its selected AppWorld prompt replaces this
interface with an \emph{Execution Ledger}, a typed
\emph{State Table}, \emph{Terminal Completion}, and
\emph{Continuation Guards}.

\begin{promptbox}{Planning rules introduced by ACON-UT}
Maintain a Completed/Pending Action Ledger with the status, evidence,
and next action for every operation.

Before any mutation, reconcile the current target set, exclusions,
authentication, required parameters, and idempotency status.

Before completion, verify that every qualifying target was processed
and that the verification set equals the expected set.

Execute the highest-priority Next Step before exploratory calls.
\end{promptbox}

These rules do more than determine which historical facts are retained.
They assign actions, impose preconditions, order execution, and define
completion. However, their planning-oriented character should not be
attributed to the modified schema itself. As shown by the
template-controlled comparison in
Appendix~\ref{app:template-controlled-acon}, ACON-UT exhibits similar
behavior when constrained to the same OpenClaw section structure as
\texttt{PAIR}. The schema determines where these rules are expressed,
but the rules originate from the trajectory-level evidence used for
adaptation.

\noindent\textbf{Planning is necessary, but plan steering is not
equivalent to compression.}
Planning is indispensable in long-horizon tasks. A useful checkpoint
must preserve completed actions, pending objectives, relevant failures,
and the state needed to choose the next action. ACON-UT explicitly
strengthens this role by turning the checkpoint into a carrier of
planning policy. Its State Table records authentication, pagination,
eligible entities, mutations, and verification, while its Continuation
Guards specify how these records should determine future actions.

This design can be effective when the induced workflow matches the
current task. A persistent action ledger can prevent completed work
from being reopened, cached API contracts can avoid repeated
documentation lookup, and explicit mutation and verification guards can
reduce duplicated or unsafe side effects. In this regime, steering the
agent through the checkpoint provides useful planning support.

The limitation is that ACON learns this steering from a
trajectory-level contrast. A successful full-context trajectory
reveals one plan that happened to succeed, while a failed compressed
trajectory reveals one plan that failed. Their difference does not
identify whether the terminal outcome changed because compression
removed necessary state or because the two stochastic executions
followed different plans. The optimizer may therefore encode
procedural patterns from the successful trajectory even when those
patterns are not the information lost during compression.

The instruction to ``Execute the highest-priority Next Step before
exploratory calls,'' for example, can prevent redundant exploration,
but it can also anchor the agent to a locally selected plan. Likewise,
reusing cached endpoint knowledge is efficient when the prior contract
remains applicable, but restrictive when a new task requires a
different tool, parameterization, or search order. These rules arise
because the trajectory-level contrast contains differences in both
retained information and downstream planning. The optimizer may
therefore convert procedural choices from a successful trajectory into
global compression rules, regardless of whether those rules are placed
in a newly created ledger or in the original OpenClaw sections.

\noindent\textbf{Trajectory-level feedback can conflate missing memory
with missing environment state.}
The original ACON-UT prompt provides a concrete example of the
attribution ambiguity introduced by trajectory-level feedback:

\begin{promptbox}{Runtime assumption introduced by ACON-UT}
Python/runtime variables do not survive a session reset. Never assume
a prior-session variable exists unless it is listed below and recreated
in the new session.
\end{promptbox}

This rule conflicts with the AppWorld execution contract:

\begin{promptbox}{AppWorld runtime contract}
You can use the variables from the previous code blocks in the
subsequent code blocks.
\end{promptbox}

A context compression removes earlier interaction text from the model's
visible history, but it does not reset the AppWorld Python runtime.
Variables created in previous code blocks therefore remain available.
Compression may cause the agent to forget a variable's name, meaning,
or relation to the task, but this is an epistemic failure of the
compressed context rather than the destruction of the underlying
runtime state. The appropriate compression correction is to preserve
the relevant variable binding, not to instruct the agent to recreate
the variable itself.

ACON-UT instead treats the observed failure as evidence that runtime
state must be reconstructed. Its State Table requests a reconstruction
instruction for each variable, while its Continuation Guards prohibit
the downstream agent from referencing prior-session variables before
they have been recreated or re-authenticated. Following these rules can
produce unnecessary variable initialization, repeated authentication,
duplicated API calls, or recomputation of results that remain available
in the environment. For state-changing operations, such repetition may
also introduce additional verification and idempotency requirements.

This incorrect rule illustrates the attribution problem of
trajectory-level adaptation. A successful full-context trajectory may
reuse a variable correctly, while a failed compressed trajectory may
lose track of that variable after compression. From the terminal
contrast alone, however, the optimizer cannot determine whether the
compressed context omitted the variable binding, whether the agent
selected a different action, or whether the runtime state itself was
unavailable. ACON-UT resolves this ambiguity by introducing a general
runtime-reconstruction rule, even though the target environment
explicitly preserves variables across code blocks.

Boundary-level counterfactual evidence provides a more appropriate
basis for this diagnosis. By comparing continuations from the same
environment state immediately before and after compression,
\texttt{PAIR} holds the runtime state fixed and isolates the context
replacement. If the post-compression continuation can no longer use a
previously created variable, the contrast identifies missing textual
state, such as the variable name, value, or purpose, rather than a
runtime reset. The corresponding prompt revision can therefore target
what the compressor must preserve without altering the environment
model supplied to the downstream agent.

The mistaken assumption therefore originates from coarse attribution,
not from the choice of checkpoint schema. Unrestricted rewriting makes
the resulting execution policy more visible by allowing it to be
encoded as a dedicated State Table requirement and Continuation Guard,
but the same assumption could also be written into a fixed
\emph{Critical Context} or \emph{Next Steps} section. Template locking
controls the interface through which the rule is expressed; it does
not remove the trajectory-level ambiguity that produced the rule.

\noindent\textbf{\texttt{PAIR} preserves planning state without fixing
a planning policy.}
Keeping the OpenClaw framework fixed does not remove planning from the
checkpoint. \texttt{PAIR} retains both \emph{Key Decisions} and
\emph{Next Steps}, but constrains how these sections are populated.

\begin{promptbox}{State-grounding rules introduced by \texttt{PAIR}}
ADD only facts supported by the new messages. Do not invent a new task,
constraint, result, API, capability, or pending action from
documentation or credentials alone.

Keep Next Steps limited to genuinely unresolved work and avoid
speculative searches, repeated documentation reads, or re-investigation
of settled facts.

Record authentication or reusable session availability when
established, without copying secret passwords or token values. Do not
make a completed login the next step unless the history shows it
failed, expired, or otherwise must be retried.
\end{promptbox}

These rules preserve the current plan as execution state: what has been
decided, what has been completed, what remains unresolved, and which
next actions are justified by the observed trajectory. They do not
supply a fixed authentication, pagination, mutation, or completion
procedure. In particular, session availability is represented according
to observed evidence rather than a universal assumption about whether
runtime state survives compression.

The distinction is therefore between preserving a plan and prescribing
one. \texttt{PAIR} requires the compressor to retain the information
needed for subsequent planning, while leaving the frozen downstream
agent responsible for revising that plan as new observations arrive.
ACON-UT additionally uses the checkpoint to specify how the agent
should plan and act. Both approaches can influence execution, but they
intervene at different levels.

\noindent\textbf{Boundary evidence directs adaptation toward
representational failures.}
The different rules follow naturally from the evidence used by the two
optimizers. ACON observes complete successful and failed trajectories
and must infer globally what the successful execution did better. A
reasonable response to this underdetermined signal is to encode the
successful trajectory's authentication, mutation, verification, and
completion patterns as reusable guidance for future execution. This
response does not require a new checkpoint schema: the same guidance
can be expressed within an execution ledger or within fixed sections
such as \emph{Constraints}, \emph{Key Decisions}, and
\emph{Next Steps}.

\texttt{PAIR} instead compares continuations from the same environment
state before and after a particular context replacement. The optimizer
can inspect which task predicate, identifier, completion status, or
continuation-critical fact changed at that boundary. The resulting
prompt includes rules such as preserving the user's exact qualifying
subset, distinguishing final task results from intermediate aggregates,
and revising stale state only when contradicted by new observations.
These rules address how the compressor represents state rather than how
the downstream agent should solve an entire class of tasks.

Template locking alone does not guarantee this representational focus.
As the controlled ACON-UT variant demonstrates, planning and
verification policies can still be expressed within the fixed
OpenClaw sections. The key constraint in \texttt{PAIR} is instead the
combination of localized evidence and section-level revision:
counterfactual continuations identify what changed at a particular
context replacement, and the optimizer is asked to repair the
corresponding state representation. The fixed interface preserves
deployment compatibility, while boundary localization supplies the
substantive alignment with compression-induced failures.

\noindent\textbf{Connection to AppWorld performance.}
The prompt-level distinction is consistent with the AppWorld results
in Table~\ref{tab:appworld_gpt_luna}. Under history-only compression,
\texttt{PAIR} reaches $84.7\%$ accuracy and $79.8\%$ Pass$^3$,
compared with $81.2\%$ and $70.8\%$ for ACON-UT. It also requires
fewer interaction steps ($16.6$ versus $18.9$). Under
prefix-conditioned compression, \texttt{PAIR} reaches $85.9\%$
accuracy and $80.4\%$ Pass$^3$, compared with $81.3\%$ and $72.6\%$
for ACON-UT, again with fewer steps ($15.2$ versus $17.3$).

The template-controlled results in
Table~\ref{tab:template-controlled-acon} show that these gains are not
explained by the visible schema difference. When ACON-UT is constrained
to the same OpenClaw section structure, \texttt{PAIR} continues to
achieve higher accuracy and Pass$^3$ under both compression scopes.
The controlled prompts also have comparable interaction lengths and
peak context sizes. Template locking therefore removes interface
flexibility as a confound without removing the substantive distinction
between trajectory-level and boundary-level adaptation.

ACON-UT contains multiple rules intended to prevent redundant
execution, yet it requires more steps than \texttt{PAIR} in the
original AppWorld comparison. The incorrect assumption about runtime
persistence provides one concrete mechanism: reconstructing
still-available state introduces work that the environment does not
require. More generally, broad planning and verification rules can
consume interactions when they are not matched to the current task.
\texttt{PAIR} instead preserves the state required by the original
agent without replacing the environment model or prescribing a new
task-solving policy.

\noindent\textbf{Summary.}
The two prompts embody different adaptation strategies. ACON uses
trajectory-level feedback that jointly reflects information retention,
planning, and stochastic execution. Its optimizer may consequently
translate patterns from successful trajectories into general planning,
verification, or stopping rules. The original ACON proposer may also
rewrite the checkpoint schema, but the template-controlled ablation
shows that this structural freedom does not explain its behavior or the
performance difference from \texttt{PAIR}. \texttt{PAIR} instead uses
boundary-level evidence to revise how task, evidence, and planning state
are represented within a fixed OpenClaw interface.

Accordingly, \texttt{PAIR} does not remove planning from compressed
memory; it separates preservation of the current planning state from
selection of the future planning policy. The compressor records the
goal, constraints, decisions, unresolved work, and
continuation-critical facts, while the frozen downstream agent remains
responsible for adapting its plan to the current task. This design
yields higher reliability on AppWorld and stronger answer quality under
cross-task transfer, while avoiding the need to embed a plan inferred
from one set of successful trajectories into every future checkpoint.

%%%%%%%%%%%%%%%%%%%%%%%%%%%%%%%%%%%%%%%%%%%%%%%%%%%%%%%%%%%%

% \newpage
% \input{sections/checklist}

\end{document}